\documentclass{article}

\usepackage[preprint]{neurips_2026}

\usepackage[utf8]{inputenc} 
\usepackage[T1]{fontenc}    
\usepackage{hyperref}       
\usepackage{url}            
\usepackage{booktabs}       
\usepackage{amsfonts}       
\usepackage{nicefrac}       
\usepackage{microtype}      
\usepackage{xcolor}         

\usepackage{graphicx}       
\usepackage{subcaption}     
\usepackage{amsmath}        
\usepackage{amssymb}        
\usepackage{mathtools}      
\usepackage{amsthm}         

\usepackage[capitalize,noabbrev]{cleveref}

\theoremstyle{plain}
\newtheorem{theorem}{Theorem}[section]

\newtheorem{lemma}[theorem]{Lemma}

\theoremstyle{definition}

\newtheorem{assumption}[theorem]{Assumption}
\newtheorem*{remark}{Remark}

\usepackage{titletoc}
\usepackage[titletoc]{appendix}

\makeatletter
\renewcommand{\tableofcontents}{%
    \section*{\contentsname}%
    \begingroup
    \setlength{\baselineskip}{10pt} 
    \@starttoc{toc}%
    \endgroup
}
\makeatother

\allowdisplaybreaks[4]

\title{On the Blessing of Pre-training in \\ Weak-to-Strong Generalization}

\author{
Wei Yao\textsuperscript{1,4}, Wang Zhaoyang\textsuperscript{1}, Gengze Xu\textsuperscript{1}, Chen Qian\textsuperscript{1}, Dongrui Liu\textsuperscript{2}, \\
\textbf{Ziqiao Wang\textsuperscript{3}, Yong Liu\textsuperscript{1}$^{\dag}$, Yunbei Xu\textsuperscript{4}}$^{\dag}$ \\ \\
$^1$ Renmin University of China, 
$^2$ Shanghai Jiao Tong University, \\
$^3$ Tongji University, 
$^4$ National University of Singapore \\ \\
\tt\footnotesize\{wei.yao,liuyonggsai\}@ruc.edu.cn,~~yunbei@nus.edu.sg \\
}

\begin{document}

\maketitle

\begin{abstract}
The paradigm of Weak-to-Strong Generalization (W2SG) suggests that a pre-trained strong model can surpass its weak supervisor, yet the decisive role of pre-training remains theoretically and empirically under-explored. 
In this work, we identify pre-training as the essential prerequisite for the emergence of W2SG. 
Theoretically, we formalize the W2SG problem within a high-dimensional single-index model framework using spiked Gaussian data, modeling pre-training as a spectral initialization step. 
Building upon prior impossibility results regarding the failure of learning under random initialization, we prove that W2SG is achievable when pre-training provides a geometric warm start that places the model within an ``effective region'' characterized by a perturbed strong-convexity geometry. 
Within this region, we derive a rigorous generalization bound that naturally captures the optimization dynamics: an initial performance improvement followed by a saturation bottleneck dictated by the weak supervisor's bias. 
Empirically, we first validate all our assumptions and theoretical insights through controlled synthetic simulations.
Finally, through a massive-scale evaluation of hundreds of intermediate pre-training checkpoints from large language models, we demonstrate that W2SG is not an innate capability but emerges via a phase transition tightly coupled with the progression of pre-training.
\let\thefootnote\relax\footnotetext{$^{\dag}$ Corresponding author}
\end{abstract}

\section{Introduction}

As Large Language Models (LLMs) continue to scale, they are rapidly approaching, and in some domains surpassing, human-level capabilities~\citep{achiam2023gpt,team2024gemini,guo2025deepseek}. 
This widening capability gap gives rise to a fundamental alignment challenge~\citep{openai_superalignment}: how can we reliably supervise AI systems that may be much more capable than their human overseers? 
To address this, \citet{openai_wtsg,burns2023weak} proposed the paradigm of Weak-to-Strong Generalization (W2SG):

\begin{quote}
\textit{``We find that when we naively finetune strong \textbf{pre-trained} models on labels generated by a weak model, they consistently perform better than their weak supervisors, a phenomenon we call weak-to-strong generalization.''}
\end{quote}


Since the proposal of W2SG, the field has witnessed a surge of research aimed at understanding how strong models can surpass their weak supervisors. 
Empirically, researchers have characterized W2SG's fundamental properties~\citep{shin2024weak,goel2025great} and developed advanced methodologies to enhance performance~\citep{ye2025iterative,cui2024bayesian}. 
Parallel theoretical efforts have sought to pinpoint feasibility conditions for the emergence of W2SG~\citep{charikar2024quantifying,lang2024theoretical,wu2024provable}.
However, a critical factor remains less explored: the decisive role of \textit{\textbf{pre-training}}. 
This gap is likely attributable to formidable barriers in two domains: (1) empirically, tracing the evolutionary dynamics of W2SG throughout the entire pre-training trajectory of LLMs is computationally prohibitive; and (2) theoretically, characterizing how pre-training shapes the optimization landscape in high-dimensional, non-convex settings poses significant analytical challenges. 
Consequently, this brings us to the central question motivating our work: What is the fundamental impact of pre-training on the emergence of W2SG?

To answer this question, we theoretically formalize the W2SG problem within a high-dimensional single-index model framework~\citep{arous2021online} with spiked Gaussian data. This model has served as a representative setting across statistical physics, neural network dynamics, and language-model alignment~\citep{kallus2025semiparametric}. 
See Section~\ref{sec:theory} 
for why these choices provide a natural analytical model of real-world W2SG in representation learning for our purposes.
We model the unsupervised pre-training process as a spectral initialization step to characterize the fundamental mechanism of W2SG. 
Building upon prior impossibility results regarding the failure of learning under random initialization~\citep{jones-mccormick2025provable}, we prove that W2SG is achievable when pre-training provides a geometric warm start that places the model within an ``effective region.'' 
Crucially, we prove that the loss landscape within this region exhibits a perturbed strong-convexity geometry, which guarantees a strict restoring force towards the ground truth despite the presence of imperfect supervision. 
We establish that navigating this basin requires several prerequisites: a detectable latent signal in the pre-training data, sufficient intrinsic alignment with the downstream task, an informative weak supervisor, and a stable learning rate.
Under these conditions, we derive a generalization bound that naturally captures the optimization dynamics: an initial performance improvement driven by the signal pull, followed by a saturation bottleneck dictated by the weak supervisor's noise. 
Together, these results formalize why pre-training is indispensable: it drops the model inside a geometric basin of attraction, enabling the gradient to overpower the noise.

Empirically, we first validate all our assumptions and theoretical insights through controlled synthetic simulations.
We then bridge the gap to practical settings by investigating the dynamics of actual pre-training trajectory of LLMs.
Through a massive-scale evaluation of hundreds of intermediate pre-training checkpoints from the Pythia~\citep{biderman2023pythia} and OLMo~\citep{Groeneveld2023OLMo} model families, we demonstrate that W2SG emerges via a phase transition tightly coupled with the progression of pre-training. 
Crucially, as the pre-training deepens, the model transitions from a regime of inevitable W2SG failure into a successful basin, mirroring the structural prerequisites established by our theoretical bounds.
Intriguingly, this emergent phase transition echoes recent findings in the mechanistic interpretability literature regarding trustworthiness~\citep{qian-etal-2024-towards}. Motivated by this parallel, we observe a robust positive correlation between W2SG performance and the emergence of linear representations~\citep{park2024the}—specifically, the extent to which the model encodes high-level concepts as linear directions in its activation space. 
While of independent interest, this robust correlation indicates that the geometric structure progressively acquired during pre-training is fundamentally linked to the model's ability to elicit W2SG.


\section{Related Work}

The pursuit of superalignment~\citep{openai_superalignment} led to the proposal of Weak-to-Strong Generalization (W2SG)~\citep{burns2023weak}, a paradigm demonstrating that strong pre-trained models can transcend the limitations of weaker supervisory signals. 
This discovery catalyzed a wave of empirical studies examining W2SG's fundamental characteristics~\citep{shin2024weak,yang2024super,goel2025great,yao2025understanding} and extending its application to broader tasks and scenarios~\citep{guo2024vision,yang-etal-2024-weak,pawelczyk2024generalizing,zhou2025weak}. 
Simultaneously, methodological advancements have emerged to enhance strong model performance, notably through iterative refinement techniques~\citep{lyu2024macpo,ye2025iterative,lang2025debate} and the aggregation of multiple weak supervisors~\citep{agrawal2024ensemw2s,sang2024improving,liu2024co,cui2024bayesian,somerstep2025limitations}. 
Theoretical investigations have paralleled these efforts, aiming to pinpoint the feasibility conditions for W2SG. These inquiries range from deriving generalization bounds under convex assumptions~\citep{charikar2024quantifying,mulgund2025relating,yao2025understanding,yao2025revisiting} and analyzing error correction via robustness frameworks~\citep{lang2024theoretical,shin2024weak}, to rigorous characterizations under Gaussian settings involving distribution shift and intrinsic dimension~\citep{ildiz2025highdimensional,somerstep2025transfer,dong2025discrepancies,jeon2025weaktostrong,liu2026does}.
Recent work has also begun to explore feature learning and representation dynamics~\citep{xue2025representations,wu2024provable,oh2025from,moniri2025mechanisms,medvedev2025weak}.
However, to our best knowledge, the specific influence of pre-training on W2SG within high-dimensional and non-convex optimization landscapes remains insufficiently explored in both theory and practice.


We review the W2SG literature in the main paper and defer related work on the linear representation hypothesis to Appendix~\ref{related:lrh}. Recent theoretical studies also show that W2SG can arise in certain non-pretrained settings—for example, via strong regularization effects or architectural capacity. These results clarify that W2SG is statistically possible under {\it specific} modeling assumptions and regimes. Our focus is complementary: we argue that, {\it in the context of  LLMs}, pre-training is a central and practically relevant mechanism shaping when and how W2SG reliably emerges. We provide a detailed discussion in Appendix~\ref{related:discussion}  and clarify the scope and implications of each line of work.

\section{Theoretical Insights} \label{sec:theory}

In this section, we formalize the W2SG problem in a nonlinear, high-dimensional setting via the single-index model, which has become a canonical theoretical lens for analyzing the dynamics of one-hidden-layer neural networks~\citep{xu2025towards}. More recently, it has been argued that aligning language models can be viewed as implicitly learning a single-index model~\citep{kallus2025semiparametric}.
Unlike linear or kernel-based regimes, the single-index model captures the nonconvex nature of feature learning~\citep{bietti2022learning}, making it a natural analytical apparatus to study how pre-training provides the necessary warm start to escape high-dimensional noise and facilitate W2SG.

\subsection{Problem Setting} \label{sec:prob_setting}

\noindent \textbf{Data Distribution.}
To theoretically formalize the effect of pre-training, we must model the underlying geometry of the data. Standard isotropic Gaussian inputs are fundamentally insufficient, as they represent a featureless space lacking any privileged semantic directions. Instead, we model the inputs using \textit{spiked Gaussian data} $\mathbf{x} \sim \mathcal{N}(\mathbf{0}, \mathbf{\Sigma})$, where the spiked covariance matrix is defined as $\mathbf{\Sigma} = \mathbf{I}_d + \lambda \mathbf{v}\mathbf{v}^\top$ with $\lambda > 0$ and $\|\mathbf{v}\|_2 = 1$. The ``spike'' vector $\mathbf{v}$ mathematically represents the dominant latent concept extracted during pre-training. 
This modeling choice is further motivated by recent work on the linear representation hypothesis; see Appendix~\ref{related:lrh}. 
The ground-truth concept is defined as $y_{\text{gt}} = f(\mathbf{v}_*^\top \mathbf{x})$, where $\mathbf{v}_* \in \mathcal{K}$ is the target parameter vector, and $f: \mathbb{R} \to \mathbb{R}$ is a smooth and non-linear activation function. Let $z = \mathbf{v}_*^\top \mathbf{x}$ be the projection of the input along the ground-truth direction, which behaves as a scalar random variable $z \sim \mathcal{N}(0, \sigma_z^2)$ under the spiked Gaussian measure. We define the \textit{Pre-training Alignment Coefficient} $\rho := |\mathbf{v}^\top \mathbf{v}_*|$ to quantify the extent to which the pre-training data structure supports the downstream ground truth.

\noindent \textbf{Learning Objectives and Model Entities.}
To formalize the training objectives, we define the \textit{sample-wise loss} for a model $\mathbf{w}$ given a sample $(\mathbf{x}, y)$ as $\ell(\mathbf{w}; \mathbf{x}, y) = \frac{1}{2} ( f(\mathbf{w}^\top \mathbf{x}) - y )^2$.
In the W2SG setting, we define three distinct entities based on this loss:
(1) \textbf{Weak Supervisor}: A fixed model parameterized by $\mathbf{v}_{\text{weak}} \in \mathcal{K}$ that generates imperfect labels $y_{\text{weak}} = f(\mathbf{v}_{\text{weak}}^\top \mathbf{x})$. Here, $\mathbf{v}_{\text{weak}}$ is a suboptimal but informative approximation of $\mathbf{v}_*$. Specifically, we assume it maintains a strictly positive, dimension-independent geometric correlation with the ground truth: $\mathbf{v}_{\text{weak}}^\top \mathbf{v}_* \ge \gamma$ for a constant $\gamma > 0$. This ensures the supervisor provides non-trivial guidance that is strictly better than random guessing.
(2) \textbf{Strong Student Model}: A high-capacity model $\mathbf{w} \in \mathcal{K}$ trained to minimize the sample-wise loss with respect to weak labels: $\ell(\mathbf{w}; \mathbf{x}, y_{\text{weak}})$.
(3) \textbf{Strong Ceiling Model}: A hypothetical high-capacity model $\mathbf{w}_c \in \mathcal{K}$ trained directly to minimize the sample-wise loss with respect to ground-truth labels: $\ell(\mathbf{w}_c; \mathbf{x}, y_{\text{gt}})$.

\noindent \textbf{Population Risk and W2SG Definition.}
While the models optimize the sample-wise loss, our theoretical analysis relies on the population loss. We define the generalization risk with respect to the ground truth as $\mathcal{L}(\mathbf{w}) = \mathbb{E}_{\mathbf{x}}[\ell(\mathbf{w}; \mathbf{x}, y_{\text{gt}})]$.
Accordingly, the weak supervisor has a non-vanishing risk bounded away from zero ($\mathcal{L}(\mathbf{v}_{\text{weak}}) > 0$). However, its positive geometric correlation ($\gamma > 0$) guarantees its expected risk is strictly lower than that of an uninformative random state.
The student, while trained on weak labels, effectively minimizes the weak population loss $\mathcal{L}_{\text{w2s}}(\mathbf{w}) = \mathbb{E}_{\mathbf{x}}[\ell(\mathbf{w}; \mathbf{x}, y_{\text{weak}})]$.
Formally, W2SG is achieved when the student's generalization risk strictly surpasses that of its supervisor: $\mathcal{L}(\mathbf{w}) < \mathcal{L}(\mathbf{v}_{\text{weak}})$.

\noindent \textbf{Optimization Algorithm.} 
We analyze the training dynamics under spherical Projected Gradient Descent (PGD) in the online setting~\citep{arous2021online}. 
Given a stream of $N$ fresh samples, the algorithm performs $N$ iterations. The update rule for the strong student at step $t$ ($0 \le t < N$) is:
\begin{equation}
    \mathbf{w}_{t+1} = \mathcal{P}_{\mathcal{K}} \left( \mathbf{w}_t - \eta_t \mathbf{g}_t \right),
\end{equation}
where $\mathcal{K}=\mathbb{S}^{d-1}$ is a unit sphere containing $\mathbf{v}_*$, and $\mathcal{P}_{\mathcal{K}}$ is the Euclidean projection.
$\mathbf{g}_t$ is the stochastic gradient computed on the $t$-th sample $(\mathbf{x}_t, y_t)$ using the sample-wise loss as $\mathbf{g}_t = \nabla_{\mathbf{w}} \ell(\mathbf{w}_t; \mathbf{x}_t, y_t)$.
We depart from Empirical Risk Minimization (ERM)-based W2SG literature~\citep{lang2024theoretical,charikar2024quantifying,mulgund2025relating} by shifting the focus from hypothesis space bounds to algorithmic optimization trajectories.
We first introduce the following standard conditions.

\begin{assumption} \label{assum:activation_properties_spiked}

We assume the activation function $f$, the weak supervisor, and the high-dimensional scaling limits satisfy the following:

\noindent (1) \textbf{Information Exponent $k \ge 3$ and non-degeneracy}: The activation function lacks low-order informative components such that $\mathbb{E}[f'(z)] = 0$ and $\mathbb{E}[f''(z)] = 0$ for $z \sim \mathcal{N}(0,\sigma_z^2)$. 
The function is non-degenerate: $\mu_0 \triangleq \min \left( \mathbb{E}[f'(z)^2], \mathbb{E}\left[f'(z)^2 \frac{z^2}{\sigma_z^2}\right] \right) > 0$.

\noindent (2) \textbf{Bounded derivatives}: There exist positive constants $M_1,M_2,M_3$ such that uniformly for all $x \in \mathbb{R}$, $|f'(x)|\leq M_1, |f''(x)|\leq M_2, |f'''(x)|\leq M_3$.

\noindent (3) \textbf{Bounded weak population bias}: The norm difference between the weakly supervised expected gradient and the true expected gradient is bounded by a dimension-independent constant $\phi \ge 0$: $\|\nabla \mathcal{L}_{\text{w2s}}(\mathbf{w})-\nabla \mathcal{L}(\mathbf{w})\|_2 \leq \phi$.

\noindent (4) \textbf{High-dimensional scaling limits}: The step size $\eta = \frac{\delta}{d}$. 
In the high-dimensional limit $d,N \to \infty$, $\alpha=\frac{N}{d}$ is a fixed constant.
\end{assumption}

\begin{remark}
    An analytical consequence of the bounded derivatives is that the stochastic gradients are inherently controlled. As we establish in Lemma~\ref{lemma:conditional_variance} in the appendix, the expected squared norm of the stochastic gradient is bounded by $\mathbb{E}[\|\mathbf{g}_t\|_2^2] \le G d$, where $G>0$ is a constant completely determined by the activation derivative limits and the data covariance structure.
\end{remark}

Conditions (1) and (2) establish the structural and regularity properties of the target concept $f$. The information exponent $k \ge 3$~\citep{arous2021online,jones-mccormick2025provable,bietti2022learning} mathematically formalizes ``hard'' concepts lacking trivial global linear or quadratic trends. This specifically models the difficulty of W2SG tasks (e.g., complex reasoning) that cannot be learned from scratch and strictly necessitate pre-training. Concurrently, the bounded derivatives ensure a well-behaved optimization landscape, which are widely satisfied by practical modern activations such as GELU~\citep{hendrycks2016gaussian} and Softplus~\citep{glorot2011deep}.

Conditions (3) and (4) capture the scaling interplay between stochastic noise and population signals. The linear sample complexity $N = \Theta(d)$ aligns with standard W2SG experimental settings~\citep{yang2024super,yao2025understanding}. While the single-sample stochastic gradient inherently scales as $\mathcal{O}(d)$ in high dimensions~\citep{vershynin2018high}, the systematic weak supervision bias $\phi$ is evaluated over population expectations. This expectation integrates out the isotropic noise and restricts the gradient strictly to a low-dimensional subspace, reducing the bias $\phi$ to an $\mathcal{O}(1)$ constant independent of $d$. 
This formulation aligns with the standard analysis of first-order methods with inexact oracles~\citep{devolder2014first,ajalloeian2020convergence}, where the gradient is accessible only up to a bounded error $\phi$. 
A more detailed justification for all the aforementioned assumptions is provided in Appendix~\ref{justify:1}.

\subsection{The Failure of Strong Ceiling Model from Random Initialization}

First, we establish the fundamental limitations of W2SG in the absence of pre-training. Under our spiked Gaussian and hard concept setting, we demonstrate that a randomly initialized strong ceiling model fails to learn even with ground-truth labels. To formalize this, we directly adapt a recent convergence result regarding the failure of online SGD in high-dimensional single-index models.

\begin{lemma}[Adapted from Theorem 3.4 of~\citet{jones-mccormick2025provable}] \label{thm:random_1}

Given the problem setting specified in~\cref{sec:prob_setting}.
Suppose Assumption~\ref{assum:activation_properties_spiked}(1) and (4) hold. 
Consider a strong ceiling model trained directly on the ground-truth labels $y_{\text{gt}}$ via spherical PGD. If the initialization $\mathbf{w}_0$ is chosen uniformly at random from $\mathbb{S}^{d-1}$, then in the high-dimensional limit $d,N \to \infty$, the following holds with high probability: $\left| \mathbb{E}[\|\mathbf{w}_N - \mathbf{v}_*\|_2^2] - \|\mathbf{w}_0 - \mathbf{v}_*\|_2^2 \right| \to 0$.
\end{lemma}
By standard high-dimensional probability~\citep{vershynin2018high}, a random initialization $\mathbf{w}_0$ has a vanishing correlation $\mathcal{O}(d^{-1/2})$ with the target $\mathbf{v}_*$. Lemma~\ref{thm:random_1} indicates that the optimization dynamics fail to amplify this weak initial signal against stochastic noise. 
Consequently, the strong model cannot escape this random initialization trap and remains statistically in an uninformative state ($\langle \mathbf{w}_N, \mathbf{v}_* \rangle \to 0$). 
This causes its expected generalization risk to degenerate into an uninformative baseline.
In contrast, a valid weak supervisor $\mathbf{v}_{\text{weak}}$ inherently possesses a dimension-independent geometric correlation with the ground truth ($\langle \mathbf{v}_{\text{weak}}, \mathbf{v}_* \rangle \ge \gamma > 0$), which ensures its expected risk is lower than the uninformative baseline. 
Therefore, in the high-dimensional limit, the randomly initialized strong ceiling model without pre-training inherently underperforms the weak supervisor---even when trained on ground-truth labels. 
We provide a visual illustration of this geometric mechanism in Figure~\ref{fig:effective_region} in the appendix, with further empirical validation presented in the experiments section.

\subsection{The Blessing of Pre-training in W2SG}

Next, we demonstrate how pre-training enables the strong student to escape the aforementioned trap and ultimately surpass its weak supervisor. 
Since the weak supervisor introduces unavoidable bias, the optimization goal is to reach a parameter state that is geometrically close to $\mathbf{v}_*$.
We first establish the following perturbed strong convexity result under our W2SG setting.

\begin{theorem}[Proved in Appendix~\ref{appendix:local_strong_convexity_spiked}] \label{assumption_1}

Given the problem setting specified in~\cref{sec:prob_setting}.
Let Assumption~\ref{assum:activation_properties_spiked} hold.
There exists a constant $\tau \in (0, 1)$ defining a local region $\mathcal{R} = \{ \mathbf{w} \in \mathbb{S}^{d-1} : \mathbf{w}^\top \mathbf{v}_* \ge \tau \}$.
Within this local spherical cap $\mathcal{R}$, the weak-to-strong population loss landscape satisfies a perturbed strong convexity condition with respect to the ground truth $\mathbf{v}_*$:
\begin{align}\label{eq:one_point_strong}
\langle \nabla \mathcal{L}_{\text{w2s}}(\mathbf{w}), \mathbf{w} - \mathbf{v}_* \rangle \ge \mu \|\mathbf{w} - \mathbf{v}_*\|_2^2 - \phi \|\mathbf{w} - \mathbf{v}_*\|_2,
\end{align}
where $\mu = \frac{\mu_0}{2}$ denotes the strong convexity constant of the underlying ground-truth objective, and $\phi \ge 0$ bounds the systematic gradient bias introduced by the weak supervisor.
\end{theorem}

\noindent \textbf{Remark 1.}
In the noise-free setting ($\phi=0$), Eq.~\eqref{eq:one_point_strong} is closely related to the ``one-point strong convexity'' established by \citet{mei2018landscape} for single-index models around $\mathbf{v}_*$. 
However, they rely on a strict monotonicity assumption (e.g., $f'(z)>0$ everywhere). Our result significantly relaxes this to the non-degeneracy condition in Assumption~\ref{assum:activation_properties_spiked}(4), which only requires positivity in expectation. 
This relaxation, which necessitates a substantially more involved proof, aligns with the information component perspective emphasized in~\citet{arous2021online,jones-mccormick2025provable}. 
Consequently, this theorem provides a non-trivial refinement of the optimization landscape for non-convex  single-index models.

\noindent \textbf{Remark 2.}
In Appendix~\ref{appendix:bound_radius}, we derive a conservative analytical bound for $\tau$ to rigorously guarantee that the effective basin $\mathcal{R}$ maintains a macroscopic, dimension-independent radius, $\Theta(1)$. 
This $\Theta(1)$ property fundamentally explains why W2SG strictly necessitates a warm start: under random initialization, the initial correlation inevitably vanishes at a rate of $\mathcal{O}(d^{-1/2})$~\citep{vershynin2018high}, deterministically trapping the model in high-dimensional noise as $d \to +\infty$. Unsupervised pre-training acts as a strict geometric prerequisite by bridging this $\mathcal{O}(d^{-1/2})$ gap, dropping the parameters directly into the $\Theta(1)$ safety zone and thereby breaking the curse of dimensionality.
For concrete intuition, Appendix~\ref{appendix:example_radius} explicitly computes this radius across three representative activation functions. 
While our analytical derivation strictly guarantees a conservative worst-case radius ranging from $0.6^\circ$ to $6.1^\circ$ around the ground truth, our subsequent synthetic experiments demonstrate that the empirical basin can be substantially broader, extending up to approximately $50^\circ$.


We introduce a set of structural conditions to rigorously quantify the optimization trajectory.

\begin{assumption} \label{assum:pretraining_w2sg}
Let Assumption~\ref{assum:activation_properties_spiked} holds.
Given the landscape geometric constants $\tau$ and $\mu$ strictly established in~\cref{assumption_1}, we assume the pre-training data distribution, weak supervisor, and algorithmic hyperparameters satisfy the following conditions:

\noindent (1) \textbf{Detectable latent structure}: The data distribution possesses a strong principal component to be statistically detectable: $\lambda > \sqrt{\frac{1+\tau^2}{\alpha (1-\tau^2)}}$.

\noindent (2) \textbf{Intrinsic task alignment}: The latent principal direction $\mathbf{v}$ inherent in the pre-training data distribution is adequately aligned with the downstream ground truth $\mathbf{v}_*$, satisfying the intrinsic correlation condition $\rho \ge (\tau + \epsilon_0) \sqrt{\frac{\alpha \lambda^2 + 1}{\alpha \lambda^2 - 1}}$ for a small positive constant $\epsilon_0 > 0$.

\noindent (3) \textbf{Informative weak supervisor}: The systematic bias $\phi$ introduced by the weak supervisor is bounded relative to the landscape curvature, satisfying $\phi \le \mu \sqrt{\frac{1-\tau}{2}}$.

\noindent (4) \textbf{Stable learning rate}: The step-size parameter $\delta$ satisfies $\delta \le \frac{\mu(1-\tau-\epsilon_d)}{7G}$ for a positive buffer $\epsilon_d = \Theta(d^{-c})$ with $c \in (0, \frac{1}{2})$.
\end{assumption}

\begin{remark}
The logical order of these assumptions reflects the mathematical dependency of the problem. 
The inherent nature of the target concept $f$ (Assumption~\ref{assum:activation_properties_spiked}) and the intrinsic latent signal strength of the data $\lambda$ act as invariant physical properties. 
Together, they rigorously determine the geometry of the optimization landscape (i.e., the basin boundary $\tau$ and strong convexity $\mu$). 
Only after they are established can we formally constrain the external conditions for successful learning: the landscape's specific geometry dictates exactly how aligned the pre-training must be ($\rho$), the maximum tolerable weak supervision bias ($\phi$), and the appropriate step size ($\delta$) required to navigate this basin.
\end{remark}

First, the innate difficulty of the task and the activation function (Assumption~\ref{assum:activation_properties_spiked}) establish the intrinsic landscape geometry $\mu$ and $\tau$ (\cref{assumption_1}). 
Based on this landscape, Conditions (1) and (2) characterize the fundamental statistical requirements of the unsupervised pre-training data. \textbf{Condition (1)} ensures the intrinsic data signal $\lambda$ is strong enough to allow unsupervised algorithms (e.g., PCA) to escape random noise. 
\textbf{Condition (2)} simply posits that this intrinsic data structure ($\mathbf{v}$) is relevant to the downstream task ($\mathbf{v}_*$). 
Together, they rigorously guarantee that the empirically extracted concept will reliably place the model's initialization inside the $\tau$-basin.
Once initialized within this effective region, \textbf{Condition (3)} limits the weak supervisor's systematic error $\phi$. 
Finally, given the bounds on noise and bias, \textbf{Condition (4)} explicitly dictates the step size $\delta$ required to control the stochastic variance during gradient descent.
See Appendix~\ref{appendix:dis_param} for a more detailed discussion.

\noindent \textbf{Modeling Pre-training via PCA Initialization.} Although characterizing the full complexity of LLM pre-training remains a formidable challenge in deep learning theory---especially within the W2SG framework---we adopt PCA as a tractable and representative proxy. 
This choice is grounded in established findings that even advanced representation learning algorithms effectively implement PCA in relevant regimes~\citep{bourlard1988auto,baldi1989neural,rolinek2019variational}. 
Consequently, we model the unsupervised pre-training process as a spectral initialization step, extracting the top eigenvector $\hat{\mathbf{v}}_{\text{PCA}}$ from the empirical covariance matrix of $N$ pre-training samples. 
To resolve the inherent sign ambiguity of PCA ($\pm \hat{\mathbf{v}}_{\text{PCA}}$), we select the directional variant that minimizes the empirical loss on weak labels. Because a valid weak supervisor naturally aligns with the ground truth ($\langle \mathbf{v}_{\text{weak}}, \mathbf{v}_* \rangle > 0$), this simple selection asymptotically identifies the correctly oriented vector. 
Therefore, without loss of generality, we define our initialization $\mathbf{w}_0 \in \{-\hat{\mathbf{v}}_{\text{PCA}}, \hat{\mathbf{v}}_{\text{PCA}}\}$ such that it guarantees an initial positive correlation $\langle \mathbf{w}_0, \mathbf{v}_* \rangle \ge 0$.
In the following, we formally analyze the generalization error of a strong model initialized via the oriented PCA procedure and subsequently fine-tuned on weak labels.

\begin{theorem}[Proved in Appendix~\ref{proof:wtsg}] \label{thm:wtsg}

Given the problem setting specified in~\cref{sec:prob_setting}.
Suppose Assumptions~\ref{assum:activation_properties_spiked} and \ref{assum:pretraining_w2sg} hold. Let $\tau \in (0,1)$ be the correlation threshold establishing the effective region $\mathcal{R}$ in~\cref{assumption_1}. 
Under spherical PGD, the optimization trajectory never escapes the effective region $\mathcal{R}$ with high probability. Furthermore, in the high-dimensional limit $d,N \to \infty$, the expected distance of the strong student to the ground truth after $T$ steps satisfies:
\begin{equation} \label{eq:w2sg_thm2}
\mathbb{E}[\|\mathbf{w}_T - \mathbf{v}_*\|_2^2] \le \left(1 - \frac{\delta \mu}{d}\right)^T\underbrace{(2 - 2\tau)}_{D_0} + \left[1 - \left(1 - \frac{\delta \mu}{d}\right)^T\right] \underbrace{\left( \frac{7 \delta G}{\mu} + \frac{2\phi^2}{\mu^2} \right)}_{D_{\infty}\leq D_0-\epsilon_d} + \tilde{\mathcal{O}}(d^{-1/2}).
\end{equation}
\end{theorem}

\begin{remark}
A major difficulty is to show that the entire iterative trajectory, rather than only the initialization point, remains within the local region satisfying Eq.~\eqref{eq:one_point_strong} in~\cref{assumption_1}; see also the numerical calculation mentioned in our experiment. Accordingly, both the bound and the core technical challenge are novel and intrinsic to the paper’s overall technical development.
\end{remark}

Our theoretical bound explicitly characterizes the fundamental geometric advantage conferred by pre-training. 
Specifically, the expected error bound in Theorem~\ref{thm:wtsg} algebraically behaves as a convex combination of the initial boundary $D_0=2-2\tau$ and the stationary error bottleneck $D_\infty = \frac{7 \delta G}{\mu} + \frac{2\phi^2}{\mu^2}$.
Based on Assumption~\ref{assum:pretraining_w2sg}(3-4), we establish that $D_\infty \le D_0-\epsilon_d$. This implies that while the optimization starts at the boundary of the effective region ($D_0$), its theoretical endpoint ($D_\infty$) is situated strictly within it. 
Given the linear sample complexity $T = N = \alpha d$, the combination coefficient $\left(1 - \frac{\delta\mu}{d}\right)^{T}$ decreases monotonically with respect to $T$, and $\lim_{d \to \infty} \left(1 - \frac{\delta\mu}{d}\right)^{T} = e^{-\alpha \delta \mu}$. 
As a result of this convex combination, the monotonically decaying coefficient ensures that the optimization trajectory is continuously pulled from the initial boundary toward the internal stationary bottleneck. 
Consequently, the strong model's expected error is mathematically guaranteed to remain geometrically constrained inside the local effective region defined in~\cref{assumption_1}:
\begin{equation*}
    \mathbb{E}[\|\mathbf{w}_T - \mathbf{v}_*\|_2^2] \le 2 - 2\tau + \Theta(d^{-c}) + \tilde{\mathcal{O}}(d^{-1/2}), \quad \text{where} \lim_{d \to +\infty} \left( \Theta(d^{-c}) + \tilde{\mathcal{O}}(d^{-1/2}) \right) = 0.
\end{equation*}
While these mathematical constraints are formulated primarily for analytical clarity, they provide the core geometric intuition underlying W2SG. In contrast to random initialization, where high-dimensional noise typically causes optimization failure, unsupervised pre-training provides a crucial structural initialization. Its effectiveness progressively shapes the optimization dynamics into two distinct regimes:
(1) \textbf{Insufficient pre-training:} In the early stages of pre-training, the initial alignment with the target task remains weak. The model is initialized outside the well-conditioned region, where high-dimensional noise dominates. In this regime, surpassing the weak supervisor is mathematically unattainable (Lemma~\ref{thm:random_1}).
(2) \textbf{Sufficient pre-training:} Once pre-training provides sufficient alignment, the model is successfully positioned inside the geometric ``safe zone'' ($\mathcal{R}$). The weak supervisor, due to its limited architectural capacity or insufficient pre-training, inherently carries specific systematic errors and therefore cannot be guaranteed to reside within such an optimal region. In contrast, because the pre-trained strong model securely operates inside $\mathcal{R}$, the landscape geometry prevents the error from diverging and allows the strong model's final performance to decouple from the weak supervisor's intrinsic limitations, thereby providing the fundamental mechanism for W2SG to successfully emerge.
We will validate these theoretical mechanisms in the subsequent sections through synthetic experiments and LLM evaluations spanning hundreds of intermediate pre-training checkpoints through months of intensive computation.

\section{How Pre-training Enables W2SG}

\subsection{Synthetic Experiments}
\label{sec:synthetic_experiments}

To empirically validate our theoretical framework, we first conduct synthetic experiments on spiked Gaussian data. We design two distinct experiments to validate our theoretical results in Section~\ref{sec:theory}. 
Detailed configurations are deferred to Appendix~\ref{app:synthetic_details}.

\textbf{Experiment 1: The Impact of Intrinsic Task Alignment.} Following the empirical finding that early stopping is crucial for eliciting W2SG \citep{burns2023weak}, we choose a short fine-tuning phase ($T=30$) and vary the intrinsic task alignment $\rho = |\mathbf{v}^\top \mathbf{v}_*|$ to observe the generalization error.
As shown in~\cref{fig:w2s_syn_1}, successful W2SG occurs only when the pre-training alignment $\rho$ surpasses a critical threshold. This confirms that high-quality pre-training is a prerequisite for the student to escape the uninformative random state and overcome systematic supervisor bias.

\textbf{Experiment 2: Optimization Trajectories Under Varying Pre-training Quality.} 
To validate how the initial alignment ($\tau$) dictates the success of W2SG as established in Theorem~\ref{thm:wtsg}, we extend the fine-tuning horizon ($T=2000$). 
As illustrated in~\cref{fig:w2s_syn_2}, a poor initialization ($\tau=0.1$) prevents the model from surpassing the weak supervisor. In contrast, superior initializations (e.g., $\tau \ge 0.7$) enable a rapid error reduction that penetrates the weak baseline, before eventually saturating at a stationary limit governed by the weak label noise.
The empirical behaviors closely correspond to our theoretical dichotomy: high-quality pre-training properly positions the model to rapidly surpass the weak supervisor's baseline. Once this breakthrough occurs, the model reaches an optimal performance ceiling, which is consistent with the stationary limit bounded by systematic errors in our theory. 
Notably, the subsequent degradation—where excessive iterations cause the error to slowly regress toward the weak supervisor's level—is a widely documented empirical phenomenon \citep{burns2023weak,yao2025understanding}. This indicates that prolonged fine-tuning ultimately leads to overfitting on the weak model's systematic noise, thereby losing the inherent generalization advantage and highlighting the necessity of early stopping~\citep{medvedev2025weak}. 
Conversely, and exactly as predicted by our failure condition, poor pre-training leaves the model outside the well-conditioned basin, causing it to stagnate at or above the supervisor's error floor throughout the entire process.

We provide two conceptual illustrations in Appendix~\ref{appendix:illustration} (Figure~\ref{fig:effective_region} and Figure~\ref{fig:drift_bias}).
We further verify Assumption~\ref{assum:pretraining_w2sg}(1-4) in Appendix~\ref{appendix:syn_exp_validate_assumption} and compute the parameters $\phi, \mu$ in Appendix~\ref{app:synthetic_additional}.

\begin{figure*}[t]
\centering
\begin{subfigure}{0.49\textwidth}
    \centering
    \includegraphics[width=\textwidth]{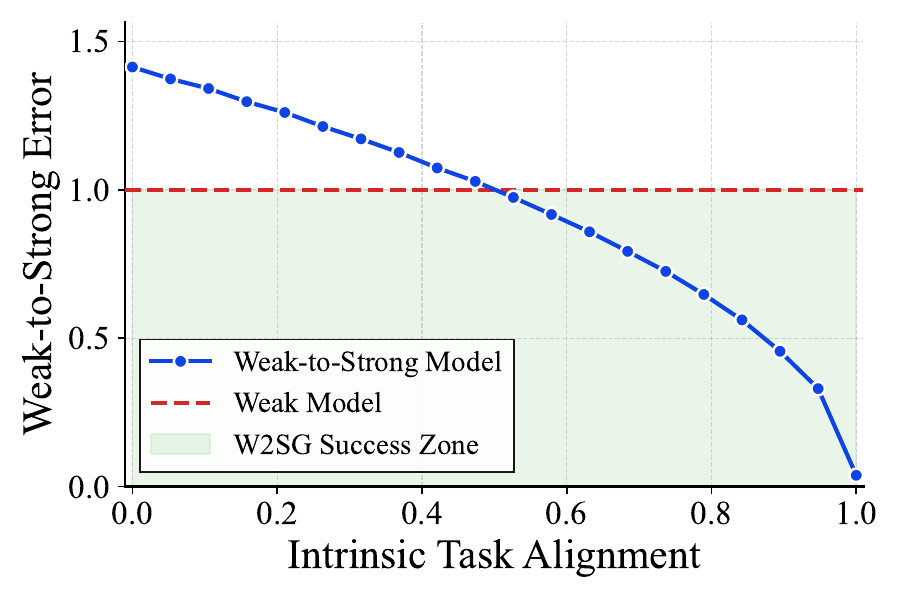}
    \caption{Experiment 1}
    \label{fig:w2s_syn_1}
\end{subfigure}
\begin{subfigure}{0.49\textwidth}
    \centering
    \includegraphics[width=\textwidth]{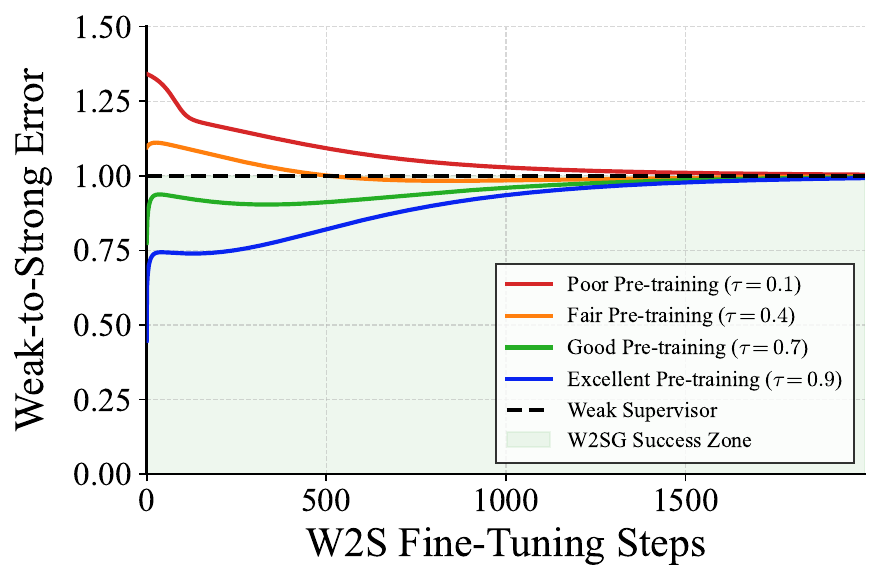}
    \caption{Experiment 2}
    \label{fig:w2s_syn_2}
\end{subfigure}
\caption{(1) \textbf{Experiment 1}: Final distance to the ground truth across varying intrinsic task alignments ($\rho$). The boundaries represent two theoretical extremes: at $\rho=0$, the lack of relevant pre-training signal reduces the model to a random initialization, stagnating at an orthogonal distance ($\|\mathbf{w} - \mathbf{v}_*\|_2 \approx \sqrt{2} \approx 1.414$ on the unit sphere). Conversely, $\rho=1$ represents perfect latent alignment, allowing the strong model to achieve near-zero error.
(2) \textbf{Experiment 2}: The optimization trajectory during W2S fine-tuning across different pre-training qualities ($\tau$).
}
\label{fig:w2s_syn}
\vspace{-10pt}
\end{figure*}

\subsection{W2SG across the Whole Pre-training}

In this section, we provide an extensive empirical evaluation of W2SG using \textit{hundreds of intermediate checkpoints} from the Pythia and OLMo suites—a large-scale analysis representing months of cumulative computation. We track the performance of W2SG across the entire pre-training trajectory to characterize its emergence and saturation patterns. Furthermore, we investigate the empirical correlation between these macroscopic trends and the linear representation hypothesis (LRH)~\citep{park2024the}, providing insights into how the geometric evolution of LLM representations aligns with the development of weak-to-strong generalization capabilities.

\noindent \textbf{Experimental Setup.}
Consistent with established W2SG benchmarks~\citep{burns2023weak,yang2024super}, we evaluate performance on reward modeling tasks covering two primary alignment dimensions: Harmlessness (CAI-Harmless~\citep{bai2022constitutional}) and Helpfulness (HH-RLHF~\citep{bai2022training}).
To comprehensively trace the evolutionary dynamics of W2SG, our study spans an extensive range of intermediate pre-training checkpoints from two distinct model families: Pythia-6.9B (154 checkpoints) and OLMo-7B (317 checkpoints).
For each task, we partition the data into ground-truth, weak supervision, and test sets.
Strong ceiling models are fine-tuned on ground-truth labels, while weak-to-strong models are trained on labels generated by a weak supervisor (e.g., Pythia-1B and OLMo-1B) for a single epoch to mitigate overfitting. 
We report the test accuracy using a linear projection head attached to the frozen backbone. Detailed configurations regarding dataset partitioning and optimization hyperparameters are provided in Appendix~\ref{appendix:exp_setup}.

\noindent \textbf{Main results.}
As illustrated in Figure~\ref{fig:w2s_helpful}, the performance of the weak-to-strong learner exhibits a clear macroscopic trend: it remains poor during the early checkpoints, progressively improves as pre-training continues, and ultimately saturates into a phase of stochastic fluctuation. 
This behavior is remarkably consistent across diverse architectures and datasets (see Appendix~\ref{appendix:extended_w2sg} for CAI-Harmless results).
The peak accuracy achieved in our experiments is competitive with, or even superior to, the baselines reported in prior studies~\citep{burns2023weak,yao2025revisiting,yang2024super}, validating the effectiveness of our experimental framework.

Crucially, these empirical observations directly echo the theoretical dichotomy established in Section~\ref{sec:theory}:
(1) \textbf{Insufficient pre-training:} In the nascent stages (e.g., the first 15 checkpoints of Pythia-6.9B), the model's latent representations lack meaningful alignment with the downstream task. Analogous to the high-dimensional noise stagnation described in \cref{thm:random_1}, W2SG essentially fails to emerge.
(2) \textbf{Sufficient pre-training:} As pre-training progresses and provides sufficient structural alignment, the model enters a well-conditioned space where optimization can effectively mitigate supervision noise. W2SG successfully emerges, but consistent with the stationary limit predicted by \cref{thm:wtsg}, it does not achieve perfect convergence. 
Instead, the performance eventually oscillates within a bounded neighborhood.
The persistent fluctuation at later checkpoints suggests that the strong model ultimately becomes trapped by the inherent noise floor of the weak labels.

\noindent \textbf{Mechanistic Insights via Linear Representation Hypothesis.}
To further explore the geometric evolution accompanying these macroscopic trends, we investigate the linear separability~\citep{park2024the,jiang2024on} of the model's internal representations during pre-training. As extensively detailed in Appendix~\ref{appendix:lrh}, we find a statistically significant positive correlation between the linear probe accuracy of the pre-trained checkpoints and their ultimate W2SG performance (e.g., achieving an $R^2 \approx 0.90$ for Pythia on CAI-Harmless, see Figure~\ref{fig:correlation_cai} and Table~\ref{tab:correlation_stats}). Furthermore, our quantitative analysis in Table~\ref{tab:correlation_new} demonstrates that this linear decodability serves as a more predictive indicator of W2SG feasibility than standard pre-training evaluation loss. This suggests that the emergence of W2SG is closely associated with the linear decodability of the pre-trained manifold: the capability to properly mitigate supervision noise and surpass the weak supervisor becomes empirically viable when the latent geometric structure of the target concepts becomes sufficiently linearly separable.

\begin{figure*}[t]
\centering
\begin{subfigure}{\textwidth}
    \centering
    \includegraphics[width=0.75\textwidth]{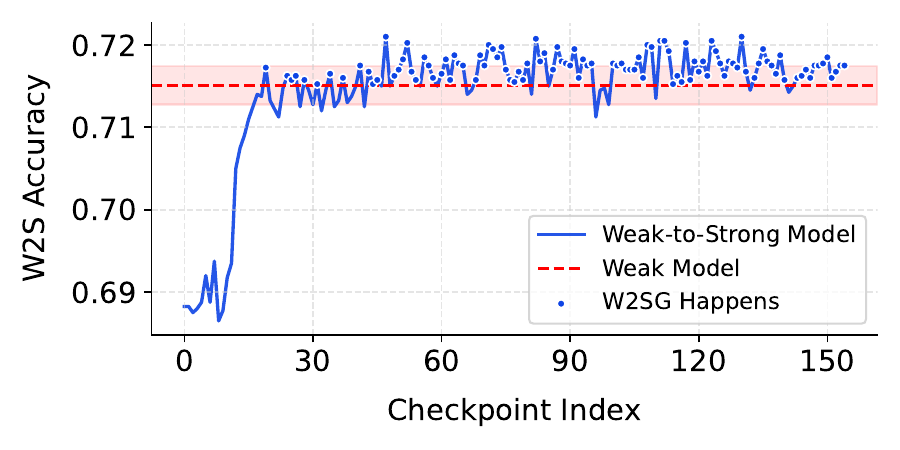}
    \caption{Pythia-6.9B}
\end{subfigure}
\vspace{20pt}
\begin{subfigure}{\textwidth}
    \centering
    \includegraphics[width=0.8\textwidth]{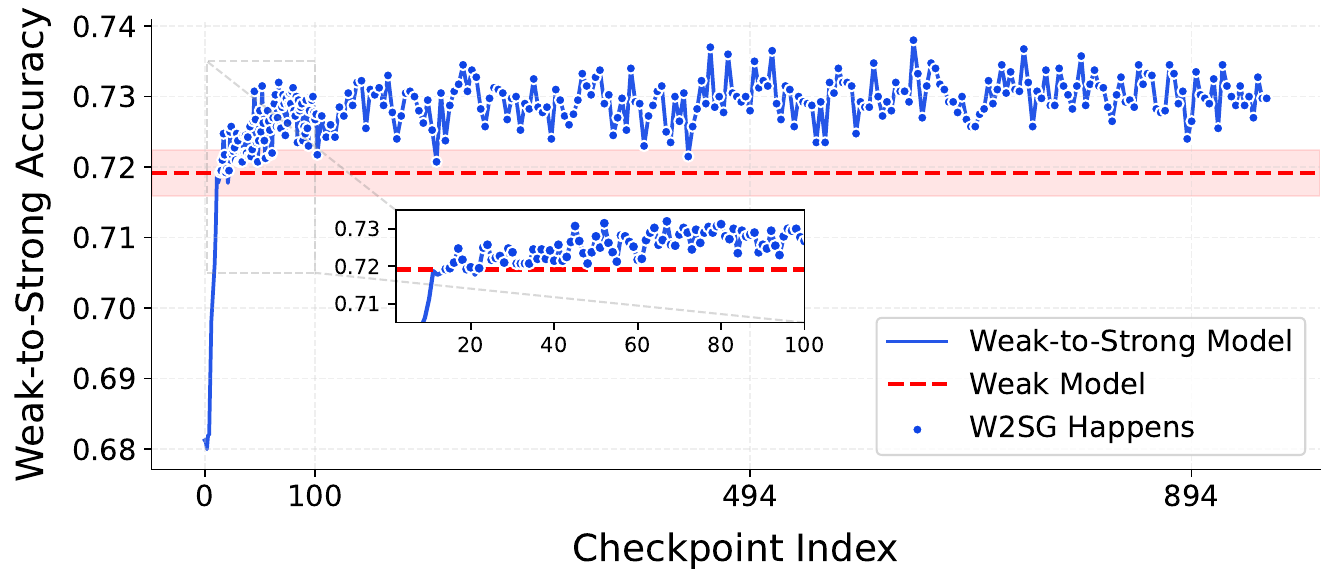}
    \caption{OLMo-7B}
\end{subfigure}
\vspace{-30pt}
\caption{Evolution of W2SG dynamics during pre-training on the HH-RLHF dataset. 
The plot reports the test accuracy of the strong student model fine-tuned on weak labels. 
The solid curve represents the mean performance aggregated over 5 random seeds, while the shaded region indicates the standard deviation. 
The dashed horizontal line denotes the weak supervisor's performance level. 
}
\vspace{-15pt}
\label{fig:w2s_helpful}
\end{figure*}

\section{Conclusion}

In this paper, we establish that pre-training is fundamentally necessary for the emergence of W2SG.
Theoretically, without pre-training, a strong model fails to learn from a weak supervisor because it becomes trapped by high-dimensional stochastic noise. Pre-training resolves this by providing a ``geometric warm start,'' which situates the model within an effective region where its gradients can successfully overpower the weak supervisor's systematic bias. 
Extensive synthetic experiments and evaluations across hundreds of large language model checkpoints corroborate these findings, demonstrating that W2SG is not an innate capability. Rather, it emerges via a distinct phase transition that is tightly coupled with the progression of pre-training.




\section*{Impact Statement}

This paper aims to advance the field of machine learning, specifically within the domain of AI alignment and safety. Our work investigates the mechanisms behind weak-to-strong generalization (W2SG), a critical paradigm for supervising future AI systems that may surpass human capabilities (superalignment). By theoretically characterizing the necessity of pre-training and the dynamics of generalization, our findings contribute to a foundational understanding of how to reliably align superhuman models using limited human supervision.
While W2SG techniques could theoretically be repurposed to amplify undesirable behaviors if the weak supervisor is malicious, the primary motivation and application of this research are to ensure the safe and controllable deployment of advanced AI systems. We believe that understanding the geometric interplay between pre-training and alignment is a vital step toward mitigating long-term risks associated with unaligned artificial general intelligence.

\section*{Limitations}

While our theoretical framework rigorously establishes the necessity of pre-training for weak-to-strong generalization (W2SG), our analysis naturally relies on modeling unsupervised pre-training as a spectral initialization step (PCA) under spiked Gaussian data assumptions. 
Although this serves as a canonical and mathematically tractable proxy in high-dimensional statistical physics, it may not encapsulate the full spectrum of non-linear feature learning dynamics present in modern, ultra-large architectures. 
Empirically, while our study involves a massive-scale evaluation across hundreds of intermediate checkpoints requiring substantial computational resources, our large language model validations are primarily focused on reward modeling tasks concerning harmlessness and helpfulness. 
Consequently, a natural avenue for future work involves extending these empirical investigations to encompass broader domains—such as complex mathematical reasoning or code generation—to further map the precise boundaries of the W2SG phase transition across various frontier models.


\bibliography{ref}
\bibliographystyle{unsrtnat}

\newpage
\tableofcontents

\newpage
\appendix
\onecolumn

{\LARGE \centering \textbf{Appendix} \par}

\section{More Related Work and Discussion}

\subsection{Linear Representation Hypothesis} \label{related:lrh}

Recent advancements in mechanistic interpretability have centered on the Linear Representation Hypothesis (LRH)~\citep{park2024the}, which posits that high-level semantic concepts are encoded as distinct linear directions within the activation spaces of LLMs. 
This perspective extends earlier findings from the era of smaller-scale models, particularly as documented in the linear probing~\citep{alain2016understanding, tenney2019you} and word embedding literature~\citep{mikolov2013linguistic,bolukbasi2016man}. 
By isolating these ``concept directions,'' researchers can implement latent space steering—a technique that modulates model behavior through the algebraic manipulation of these activation vectors~\citep{turner2023steering, zou2023representation, li2023inferencetime, gurnee2024language, marks2024the, qian-etal-2024-towards}.
Beyond empirical observation, recent theoretical inquiries have begun to elucidate the mechanisms underlying the emergence of these linear structures in pre-trained LLMs~\citep{jiang2024on, marconato2024all, li2025task}.

In this work, we establish a formal mathematical connection between W2SG and the LRH by modeling the pre-trained representation space using a spiked Gaussian distribution. 
Extensive empirical studies have firmly established that the activation spaces of pre-trained language models are fundamentally anisotropic~\citep{ethayarajh2019contextual,gao2021simcse}, with the variance of the representations being overwhelmingly concentrated along a few dominant principal directions~\citep{timkey2021all}. 
Building on this geometric reality, if pre-training encodes high-level semantic concepts as these distinct linear directions---as posited by the LRH~\citep{park2024the}---we can mathematically isolate this structural bias through the ``spike'' vector $\mathbf{v}$ in our covariance model $\mathbf{\Sigma} = \mathbf{I}_d + \lambda \mathbf{v}\mathbf{v}^\top$ (Section~\ref{sec:prob_setting}).
This specific formulation explicitly instantiates the dominant conceptual variance extracted during pre-training, providing an analytical bridge to characterize how the emergent linear geometry governs the dynamics of W2SG. 
Beyond our theoretical modeling, we also empirically validate this connection by demonstrating a statistically positive correlation between the linear probe accuracy of pre-trained checkpoints and their ultimate W2SG performance (Appendix~\ref{appendix:lrh}), indicating that this linear geometric structure is the practical prerequisite for surpassing weak supervision.

\subsection{Further Discussion} \label{related:discussion}

In this field, the researchers are enthusiastic in actively discovering the schemes of the emergence of W2SG.
While various studies have implicitly suggested or assumed the importance of pre-training through diverse theoretical lenses, some previous work even find that W2SG can happen without pre-training in some theoretical settings. 
Our view is that pre-training is necessary in the context of LLMs. 
We note that such seemingly mismatch is indeed not a contradiction, but serves as supplementary angles of view.
From these perspectives, we introduce and analyze some literature below.

\subsubsection{Pre-training is (Potentially) Necessary}


Firstly, misfit-based analyses~\citep{charikar2024quantifying, mulgund2025relating, yao2025revisiting, yao2025understanding} primarily employ the Pythagorean inequality to demonstrate that a strong model can outperform its weak supervisor by the magnitude of their ``misfit.'' A pivotal assumption in this theoretical framework is ``realizability,'' which posits that the function class of the strong model—comprising a fine-tuning head coupled with pre-trained, fixed representations—encompasses the true labeling function. Consequently, this framework is intrinsically linked to pre-training, as the prerequisite for realizability is the acquisition of high-quality representations through rigorous pre-training. This importance is explicitly acknowledged in these works (e.g., Section 3 of~\citet{charikar2024quantifying} and Section 4.2 of~\citet{yao2025revisiting}). Even when~\citet{charikar2024quantifying} and~\citet{yao2025revisiting} relax the realizability assumption using a ``non-realizability'' bound, the resulting upper bound is still constrained by the misfit between the strong ceiling model and the labeling function, which fundamentally relies on pre-trained representations. Without such a foundation, the W2SG upper bounds become vacuous, further supporting the claim that pre-training is indispensable for the emergence of W2SG.

Secondly, theoretical frameworks utilizing spectral analysis further underscore the necessity of pre-training. For instance, intrinsic dimension analysis~\citep{dong2025discrepancies} establishes that the model's intrinsic dimension appears directly in the generalization upper bound. Given that pre-training is known to compress data representations and decrease intrinsic dimensionality, it effectively tightens these bounds. \citet{dong2025discrepancies} explicitly incorporate this as a core assumption (e.g., Definition 2.3 and Assumption C.1). Similarly, \citet{ildiz2025highdimensional} model the training of the target model as a high-dimensional ridgeless regression problem. In this regime, performance is highly sensitive to the initial alignment of feature directions and the implicit regularization of the covariance spectrum. By identifying an ``amplify-to-shrink'' phase transition in optimal surrogate labels—where specific feature directions are amplified while others are discarded to mitigate bias—the authors imply that the target model must possess a sufficiently mature, pre-trained representation to capture fine-grained signals from a weak supervisor.

Furthermore, several studies~\citep{lang2024theoretical, shin2024weak} utilize general definitions of robustness and make assumptions of strong model's hypothesis class to establish W2SG bounds. The characterization of functional neighborhoods and robustness properties (see Section 3.2 of~\citet{lang2024theoretical} and Definitions 1-2 of~\citet{shin2024weak}) inherently depends on sufficient pre-training. 
\citet{wu2024provable} adopt a feature learning framework, explicitly stating that their use of zero-mean Gaussian covariates serves as an idealized proxy for pre-trained features or representations.
\citet{xue2025representations} indicate the importance of pre-training by defining a model's knowledge through its principal representations, which are explicitly formalized as the structured information a model acquires during the pre-training phase. 
These representations form the principal kernels, which Theorem 3.8 identifies as the fundamental components that determine whether a strong model can correct a weak supervisor's errors or will simply replicate them. 
Assumption 3.7 indicate that the strong model's ability to generalize beyond weak labels depends on these internal structures being well-concentrated, a property directly resulting from effective pre-training.

Recent work by~\citet{somerstep2025transfer, somerstep2025limitations} also suggests the role of pre-training. \citet{somerstep2025transfer} introduce the convex hull assumption, which suggests the optimal target function must reside within the convex hull of the pre-trained model's latent components. This implies the model must already possess the requisite latent knowledge prior to weak supervision; otherwise, as shown by Propositions 3.2 and 3.4, naive fine-tuning merely causes the model to emulate systematic supervisory errors rather than eliciting superior internal capabilities. Thus, pre-training is essential to establish a rich representation space that allows the model to treat weak labels as noisy signals for implicit Bayesian inference. Similarly, \citet{somerstep2025limitations} utilize the ``latent concept shift'' framework to argue that success depends on the model's ability to internally infer the correct concept from its existing internal experts. According to Proposition 3.2 from~\citet{somerstep2025limitations}, if the model lacks these latent mixture components—embedded during pre-training—it cannot effectively identify or select task-specific behaviors from biased weak supervision.

In summary, while the aforementioned literature may provide indirect support for our findings, these studies primarily focus on alternative perspectives rather than the intrinsic role of pre-training itself. 
To the best of our knowledge, this work represents the first comprehensive effort to explicitly identify and characterize pre-training as a statistical necessity for the emergence of W2SG in LLMs. Unlike prior research, we dedicate a substantial portion of our analysis to rigorously establishing this dependency across both empirical and theoretical dimensions.

\subsubsection{Pre-training is Not Necessary in Some Theoretical Settings}

To our best knowledge, four theoretical papers~\citep{shin2024weak,moniri2025mechanisms,medvedev2025weak,xu2025on} demonstrate that W2SG can occur without pre-training by framing the phenomenon as a consequence of the structural or optimization advantages inherent in the student model's architecture. In models such as random feature networks~\citep{medvedev2025weak} or nonlinear CNNs~\citep{oh2025from}, the strength of the student is defined by its superior capacity—specifically, its ability to capture higher-frequency signals or more complex non-linear features that the weak teacher cannot represent. These studies prove that through mechanisms such as early stopping~\citep{medvedev2025weak}, benign overfitting~\citep{oh2025from}, regularization compensation~\citep{moniri2025mechanisms}, or the approximation of a posterior mean through bias-variance optimization~\citep{xu2025on}, a larger student model naturally prioritizes the consistent underlying data distribution over the sparse noise present in the teacher's imperfect labels. Consequently, the student filters the teacher's errors and achieves a lower test error simply by leveraging its inductive bias and size advantage to find a more optimal solution within the task's hypothesis space from a random initialization.

Our work does not contradict these theoretical findings but rather addresses a different dimension of the problem specifically relevant to large-scale empirical settings in LLMs. While the aforementioned studies provide a rigorous proof-of-concept for the statistical possibility of W2SG in supervised learning, our paper identifies pre-training as the decisive catalyst that enables this phenomenon to manifest effectively in high-dimensional, complex real-world tasks where random initialization often falls into optimization traps. We characterize the role of pre-training as a necessary spectral initialization that facilitates a phase transition from a bias regime to a drift regime, ensuring that the model can actually recover the ground truth under noisy supervision. Thus, our research serves as a crucial reminder to the community that while W2SG is theoretically possible through architectural capacity alone, pre-training is the practical blessing that ensures its reliability and success in the development of LLMs.

\section{Theoretical Appendix}

\subsection{Justification of Assumption~\ref{assum:activation_properties_spiked}} \label{justify:1}

Conditions (1) and (2) impose standard structural and regularity properties on the target concept $f$ and the resulting optimization landscape. Condition (1) mathematically formalizes the difficulty of the task via the information exponent: the vanishing first and second derivatives correspond to an information exponent of $k \ge 3$~\citep{arous2021online,jones-mccormick2025provable}. While simpler activations like ReLU correspond to $k=1$ (or $k=2$ upon centering), the assumption of $k \ge 3$ is a standard condition in theoretical high-dimensional analysis to characterize the learning thresholds of complex, non-linear concepts~\citep{bietti2022learning}. Geometrically, this captures ``hard'' concepts lacking trivial global linear or quadratic trends that could be easily detected by a randomly initialized probe. In the context of W2SG, this models the difficulty of acquiring sophisticated capabilities (e.g., complex reasoning or coding features) that cannot be trivially learned from scratch but must instead be elicited via pre-training. The non-vanishing odd component is crucial for resolving the sign ambiguity of spectral initialization, ensuring the model can be correctly oriented. Meanwhile, Condition (2) provides fundamental smoothness constraints.
These regularity assumptions are 
satisfied by a broad class of functions. For instance, shifted polynomial activations (e.g., based on Hermite polynomials like $f(z) = z^3 - 3z$) strictly satisfy $k=3$ and possess constant third derivatives, trivially satisfying self-concordance. 
Furthermore, practical smooth activations prevalent in modern architectures, such as GELU~\citep{hendrycks2016gaussian}, Softplus~\citep{glorot2011deep}, or smooth sigmoidal functions~\citep{ramachandran2017searching}, possess bounded higher-order derivatives and non-vanishing higher-order Hermite coefficients. While they may contain low-order components in vanilla forms, their ability to represent complex nonlinearities ensures they qualitatively align with our analytical framework, especially when modeling hard tasks where high-order semantic signals dominate.

Crucially, Conditions (3) and (4) characterize the weak supervisor and the high-dimensional geometry, highlighting the fundamental difference between population signals and stochastic noise. The assumption of linear sample complexity $N = \Theta (d)$ in Condition (4) is consistent with the experimental setups in existing W2SG literature~\citep{yang2024super,yao2025understanding}, where both the training sample size and the model dimension are typically of the same order of magnitude. 
Condition (3) bounds the systematic gradient bias $\phi$ introduced by the weak supervisor. Because this bias evaluates the difference between \textit{population expectations}, the isotropic high-dimensional noise is integrated out over the Gaussian measure, restricting the expected gradient vectors strictly to the low-dimensional subspace spanned by the model parameters and the target concepts. Consequently, $\phi$ reduces to an $\mathcal{O}(1)$ geometric constant independent of $d$. 

While standard high-dimensional probability theory suggests that the squared norm of vectors drawn from $\mathcal{N}(\mathbf{0}, \mathbf{\Sigma})$ naturally concentrates around $\mathcal{O}(d)$ \citep{vershynin2018high}, the complex dependency of the stochastic gradient $\mathbf{g}_t$ on the non-linear activation and the weak supervision signal requires a dedicated analysis. We rigorously formalize this bound and explicitly derive its absolute constant in the following lemma:

\begin{lemma} \label{lemma:conditional_variance}
Under the problem setting in Section~\ref{sec:prob_setting} and Assumption~\ref{assum:activation_properties_spiked}, given the filtration $\mathcal{F}_{t-1}=\{\mathbf{w}_0, \mathbf{x}_0, \ldots, \mathbf{x}_{t-1}\}$ generated by the history up to step $t-1$ and $d \to +\infty$, there exists a constant $G > 0$ such that
\begin{align*}
    & \mathbb{E}[\|\mathbf{g}_t\|_2^2 \mid \mathcal{F}_{t-1}] \le G d,
    \\ & \mathbb{E}[\|\mathbf{g}_t\|_2^2] \le G d.
\end{align*}
\end{lemma}

\begin{proof}
At step $t$, conditional on the filtration $\mathcal{F}_{t-1}$, the parameter $\mathbf{w}_t$ is deterministic. The stochastic gradient evaluated on a fresh sample $(\mathbf{x}_t, y_t)$ with weak supervision $y_t = f(\mathbf{v}_{\text{weak}}^\top \mathbf{x}_t)$ is given by:
\begin{equation}
    \mathbf{g}_t = \left( f(\mathbf{w}_t^\top \mathbf{x}_t) - f(\mathbf{v}_{\text{weak}}^\top \mathbf{x}_t) \right) f'(\mathbf{w}_t^\top \mathbf{x}_t) \mathbf{x}_t.
\end{equation}
We evaluate the squared Euclidean norm of this gradient. By Assumption~\ref{assum:activation_properties_spiked}, the derivative is uniformly bounded: $|f'(z)| \le M_1$. By the Mean Value Theorem, this strictly implies that the activation function $f$ is $M_1$-Lipschitz continuous, yielding $|f(\mathbf{w}_t^\top \mathbf{x}_t) - f(\mathbf{v}_{\text{weak}}^\top \mathbf{x}_t)| \le M_1 |\mathbf{w}_t^\top \mathbf{x}_t - \mathbf{v}_{\text{weak}}^\top \mathbf{x}_t|$. Applying these bounds, we have:
\begin{align}
    \|\mathbf{g}_t\|_2^2 &= \left( f(\mathbf{w}_t^\top \mathbf{x}_t) - f(\mathbf{v}_{\text{weak}}^\top \mathbf{x}_t) \right)^2 \left( f'(\mathbf{w}_t^\top \mathbf{x}_t) \right)^2 \|\mathbf{x}_t\|_2^2 \notag \\
    &\le \left( M_1 \left| \mathbf{x}_t^\top (\mathbf{w}_t - \mathbf{v}_{\text{weak}}) \right| \right)^2 \cdot M_1^2 \|\mathbf{x}_t\|_2^2 \notag \\
    &= M_1^4 \left( \mathbf{u}_t^\top \mathbf{x}_t \right)^2 \|\mathbf{x}_t\|_2^2,
\end{align}
where we define the parameter difference vector $\mathbf{u}_t \triangleq \mathbf{w}_t - \mathbf{v}_{\text{weak}}$. Since both models are constrained to the unit sphere ($\mathbf{w}_t, \mathbf{v}_{\text{weak}} \in \mathbb{S}^{d-1}$), the triangle inequality establishes a deterministic bound on its norm: $\|\mathbf{u}_t\|_2 \le \|\mathbf{w}_t\|_2 + \|\mathbf{v}_{\text{weak}}\|_2 \le 2$. 

Taking the conditional expectation with respect to $\mathcal{F}_{t-1}$, we note a crucial measurability property: the parameter $\mathbf{w}_t$ (and thus $\mathbf{u}_t$) is entirely determined by the history and is strictly $\mathcal{F}_{t-1}$-measurable. Conversely, the fresh sample $\mathbf{x}_t \sim \mathcal{N}(\mathbf{0}, \mathbf{\Sigma})$ is strictly independent of $\mathcal{F}_{t-1}$. Distributing the conditional expectation over the sum $\|\mathbf{x}_t\|_2^2 = \sum_{i=1}^d (\mathbf{x}_t^{(i)})^2$ yields:
\begin{equation}
    \mathbb{E} \left[ (\mathbf{u}_t^\top \mathbf{x}_t)^2 \|\mathbf{x}_t\|_2^2 \mid \mathcal{F}_{t-1} \right] = \sum_{i=1}^d \mathbb{E} \left[ (\mathbf{u}_t^\top \mathbf{x}_t)^2 (\mathbf{x}_t^{(i)})^2 \mid \mathcal{F}_{t-1} \right].
\end{equation}
Because $\mathbf{x}_t$ is a zero-mean multivariate normal random vector independent of the history, we can apply Isserlis' Theorem for fourth-order Gaussian moments conditionally. Letting $A = \mathbf{u}_t^\top \mathbf{x}_t$ and $B = \mathbf{x}_t^{(i)}$, the expectation expands exactly as:
\begin{align}
    \mathbb{E}[A^2 B^2 \mid \mathcal{F}_{t-1}] = \mathbb{E}[A^2 \mid \mathcal{F}_{t-1}]\mathbb{E}[B^2 \mid \mathcal{F}_{t-1}] + 2\left( \mathbb{E}[AB \mid \mathcal{F}_{t-1}] \right)^2.
\end{align}
We systematically compute each conditional second-order moment. Crucially, because $\mathbf{u}_t$ is $\mathcal{F}_{t-1}$-measurable, it factors out of the conditional expectation as a constant vector:
\begin{enumerate}
    \item $\mathbb{E}[(\mathbf{u}_t^\top \mathbf{x}_t)^2 \mid \mathcal{F}_{t-1}] = \mathbf{u}_t^\top \mathbb{E}[\mathbf{x}_t \mathbf{x}_t^\top \mid \mathcal{F}_{t-1}] \mathbf{u}_t = \mathbf{u}_t^\top \mathbb{E}[\mathbf{x}_t \mathbf{x}_t^\top] \mathbf{u}_t = \mathbf{u}_t^\top \mathbf{\Sigma} \mathbf{u}_t$.
    \item $\mathbb{E}[(\mathbf{x}_t^{(i)})^2 \mid \mathcal{F}_{t-1}]$ represents the variance of the $i$-th component, corresponding to the diagonal entry $\mathbf{\Sigma}_{ii}$.
    \item $\mathbb{E}[(\mathbf{u}_t^\top \mathbf{x}_t) \mathbf{x}_t^{(i)} \mid \mathcal{F}_{t-1}] = \mathbb{E} \left[ \left(\sum_{j=1}^d \mathbf{u}_{tj} \mathbf{x}_t^{(j)}\right) \mathbf{x}_t^{(i)} \mid \mathcal{F}_{t-1} \right] = \sum_{j=1}^d \mathbf{u}_{tj} \mathbb{E}[\mathbf{x}_t^{(j)} \mathbf{x}_t^{(i)} \mid \mathcal{F}_{t-1}] = \sum_{j=1}^d \mathbf{u}_{tj} \mathbf{\Sigma}_{ji}$. This is exactly the $i$-th component of the matrix-vector product, denoted as $(\mathbf{\Sigma} \mathbf{u}_t)_i$.
\end{enumerate}

Substituting these identities back into the summation yields:
\begin{align}
    & \sum_{i=1}^d \mathbb{E} \left[ (\mathbf{u}_t^\top \mathbf{x}_t)^2 (\mathbf{x}_t^{(i)})^2 \mid \mathcal{F}_{t-1} \right] 
    \notag \\ = & \sum_{i=1}^d \left( \mathbb{E}[(\mathbf{u}_t^\top \mathbf{x}_t)^2 \mid \mathcal{F}_{t-1}] \cdot \mathbb{E}[(\mathbf{x}_t^{(i)})^2 \mid \mathcal{F}_{t-1}] + 2 \left( \mathbb{E}[(\mathbf{u}_t^\top \mathbf{x}_t) \mathbf{x}_t^{(i)} \mid \mathcal{F}_{t-1}] \right)^2 \right)
    \notag \\ = & \sum_{i=1}^d \left( (\mathbf{u}_t^\top \mathbf{\Sigma} \mathbf{u}_t) \mathbf{\Sigma}_{ii} + 2 (\mathbf{\Sigma} \mathbf{u}_t)_i^2 \right).
\end{align}
We evaluate the sum over the dimension $d$. The first term factors out the quadratic form, leaving $\sum_{i=1}^d \mathbf{\Sigma}_{ii} = \text{tr}(\mathbf{\Sigma})$. The second term is exactly the squared Euclidean norm of the vector $\mathbf{\Sigma} \mathbf{u}_t$, which evaluates to $\sum_{i=1}^d (\mathbf{\Sigma} \mathbf{u}_t)_i^2 = \|\mathbf{\Sigma} \mathbf{u}_t\|_2^2 = \mathbf{u}_t^\top \mathbf{\Sigma}^\top \mathbf{\Sigma} \mathbf{u}_t = \mathbf{u}_t^\top \mathbf{\Sigma}^2 \mathbf{u}_t$. Therefore:
\begin{align}
    \mathbb{E} \left[ (\mathbf{u}_t^\top \mathbf{x}_t)^2 \|\mathbf{x}_t\|_2^2 \mid \mathcal{F}_{t-1} \right] &= \text{tr}(\mathbf{\Sigma}) (\mathbf{u}_t^\top \mathbf{\Sigma} \mathbf{u}_t) + 2 \mathbf{u}_t^\top \mathbf{\Sigma}^2 \mathbf{u}_t.
\end{align}

Under the spiked covariance model $\mathbf{\Sigma} = \mathbf{I}_d + \lambda \mathbf{v}\mathbf{v}^\top$ (where $\|\mathbf{v}\|_2=1$), the trace is computed explicitly as $\text{tr}(\mathbf{\Sigma}) = \text{tr}(\mathbf{I}_d) + \lambda \text{tr}(\mathbf{v}\mathbf{v}^\top) = d + \lambda\|\mathbf{v}\|_2^2 = d + \lambda$. Furthermore, the maximum eigenvalue is $(1+\lambda)(\mathbf{\Sigma}) = 1 + \lambda$. Consequently, applying the Rayleigh quotient property and recalling $\|\mathbf{u}_t\|_2^2 \le 4$, the quadratic forms are strictly bounded by:
\begin{align*}
    & \mathbf{u}_t^\top \mathbf{\Sigma} \mathbf{u}_t \le (1+\lambda)(\mathbf{\Sigma}) \|\mathbf{u}_t\|_2^2 \le 4(1+\lambda), \\
    & \mathbf{u}_t^\top \mathbf{\Sigma}^2 \mathbf{u}_t \le (1+\lambda)(\mathbf{\Sigma}^2) \|\mathbf{u}_t\|_2^2 = (1+\lambda)^2(\mathbf{\Sigma}) \|\mathbf{u}_t\|_2^2 \le 4(1+\lambda)^2.
\end{align*}
Substituting these trace and quadratic bounds back into the gradient expectation inequality, we obtain:
\begin{align*}
    \mathbb{E}[\|\mathbf{g}_t\|_2^2 \mid \mathcal{F}_{t-1}] &\le M_1^4 \left[ (d + \lambda) \cdot 4(1+\lambda) + 2 \cdot 4(1+\lambda)^2 \right] \\
    &= M_1^4 \left[ 4d(1+\lambda) + 4\lambda(1+\lambda) + 8(1+\lambda)^2 \right] \\
    &= 4M_1^4 (1+\lambda) \cdot d + M_1^4 [4\lambda(1+\lambda) + 8(1+\lambda)^2].
\end{align*}
Because $M_1$ and $\lambda$ are strict $\mathcal{O}(1)$ constants entirely independent of the dimension $d$, as $d \to +\infty$, we can define the scale-invariant geometric constant:
\begin{equation}
    G \triangleq 4 M_1^4 (1+\lambda) + \epsilon_G > 4 M_1^4 (1+\lambda) + \lim_{d \to +\infty} \frac{M_1^4 [4\lambda(1+\lambda) + 8(1+\lambda)^2]}{d},
\end{equation}
where $\epsilon_G>0$ is a constant.
This establishes that $\mathbb{E}[\|\mathbf{g}_t\|_2^2 \mid \mathcal{F}_{t-1}] \le G d$.
Finally, by the Law of Total Expectation, the conditional second moment bound guarantees the unconditional bound: 
$$\mathbb{E}[\|\mathbf{g}_t\|_2^2] = \mathbb{E}[\mathbb{E}[\|\mathbf{g}_t\|_2^2 \mid \mathcal{F}_{t-1}]] \le G d.$$
\end{proof}

\subsection{Proof of~\cref{assumption_1}} \label{appendix:local_strong_convexity_spiked}

\begin{proof}
The variance of $z = \mathbf{v}_*^\top \mathbf{x} \sim \mathcal{N}(0, \sigma_z^2)$ is given by $\sigma_z^2 = \mathbf{v}_*^\top \mathbf{\Sigma} \mathbf{v}_* = 1 + \lambda(\mathbf{v}_*^\top \mathbf{v})^2 = 1 + \lambda \rho^2$.
The ground-truth population loss is $\mathcal{L}(\mathbf{w}) = \frac{1}{2} \mathbb{E}_{\mathbf{x}}[(f(\mathbf{w}^\top \mathbf{x}) - f(\mathbf{v}_*^\top \mathbf{x}))^2]$. 
By applying the chain rule and the product rule, the Hessian matrix of the population loss is:
\begin{equation}
    H(\mathbf{w}) \triangleq \nabla^2 \mathcal{L}(\mathbf{w}) = \mathbb{E}_{\mathbf{x}}[f'(\mathbf{w}^\top \mathbf{x})^2 \mathbf{x} \mathbf{x}^\top] + \mathbb{E}_{\mathbf{x}}[(f(\mathbf{w}^\top \mathbf{x}) - f(\mathbf{v}_*^\top \mathbf{x})) f''(\mathbf{w}^\top \mathbf{x}) \mathbf{x} \mathbf{x}^\top].
\end{equation}
At the ground-truth parameter $\mathbf{w} = \mathbf{v}_*$, the residual term exactly vanishes because $f(\mathbf{v}_*^\top \mathbf{x}) - f(\mathbf{v}_*^\top \mathbf{x}) = 0$. Thus, the Hessian simplifies to $H(\mathbf{v}_*) = \mathbb{E}_{\mathbf{x}}[f'(z)^2 \mathbf{x} \mathbf{x}^\top]$.
To compute this expectation over the anisotropic distribution $\mathcal{N}(\mathbf{0}, \mathbf{\Sigma})$, we apply the conditional Gaussian decomposition theorem. We decompose the random vector $\mathbf{x}$ conditioned on the scalar random variable $z$:
\begin{equation}
    \mathbf{x} = \frac{\text{Cov}(\mathbf{x}, z)}{\text{Var}(z)} z + \mathbf{x}_\perp = \frac{\mathbf{\Sigma} \mathbf{v}_*}{\sigma_z^2} z + \mathbf{x}_\perp = \mathbf{a} z + \mathbf{x}_\perp,
\end{equation}
where $\mathbf{a} = \frac{\mathbf{\Sigma} \mathbf{v}_*}{\sigma_z^2}$. By the properties of Gaussian conditioning, the residual vector $\mathbf{x}_\perp$ is independent of $z$, with $\mathbb{E}[\mathbf{x}_\perp] = 0$. Its covariance matrix is given by the Schur complement: $\mathbb{E}[\mathbf{x}_\perp \mathbf{x}_\perp^\top] = \mathbf{\Sigma} - \frac{\mathbf{\Sigma} \mathbf{v}_* \mathbf{v}_*^\top \mathbf{\Sigma}}{\sigma_z^2}$.

Substituting the decomposition into the outer product $\mathbf{x} \mathbf{x}^\top$ and taking the expectation with respect to the independent variables $z$ and $\mathbf{x}_\perp$, the cross terms evaluate to zero:
\begin{equation}
    H(\mathbf{v}_*) = \mathbb{E}_z[f'(z)^2 z^2] \mathbf{a} \mathbf{a}^\top + \mathbb{E}_z[f'(z)^2] \mathbb{E}_{\mathbf{x}_\perp}[\mathbf{x}_\perp \mathbf{x}_\perp^\top] = c_2 \frac{\mathbf{\Sigma} \mathbf{v}_* \mathbf{v}_*^\top \mathbf{\Sigma}}{\sigma_z^2} + c_1 \left( \mathbf{\Sigma} - \frac{\mathbf{\Sigma} \mathbf{v}_* \mathbf{v}_*^\top \mathbf{\Sigma}}{\sigma_z^2} \right),
\end{equation}
where $c_1 = \mathbb{E}[f'(z)^2]$ and $c_2 = \mathbb{E}\left[f'(z)^2 \frac{z^2}{\sigma_z^2}\right]$.

We analyze the quadratic form $\mathbf{u}^\top H(\mathbf{v}_*) \mathbf{u}$ for an arbitrary unit vector $\mathbf{u} \in \mathbb{S}^{d-1}$. Let $V = \mathbf{u}^\top \mathbf{\Sigma} \mathbf{u}$ and $W = \frac{(\mathbf{u}^\top \mathbf{\Sigma} \mathbf{v}_*)^2}{\sigma_z^2}$. The quadratic form is:
\begin{equation}
    \mathbf{u}^\top H(\mathbf{v}_*) \mathbf{u} = c_2 W + c_1 (V - W).
\end{equation}
First, due to the spiked structure $\mathbf{\Sigma} \ge \mathbf{I}_d$, we have $V = \mathbf{u}^\top \mathbf{\Sigma} \mathbf{u} \ge \|\mathbf{u}\|_2^2 = 1$.

Second, by the Cauchy-Schwarz inequality for the inner product induced by $\mathbf{\Sigma}$, we have $W \sigma_z^2 = (\mathbf{u}^\top \mathbf{\Sigma} \mathbf{v}_*)^2 \le (\mathbf{u}^\top \mathbf{\Sigma} \mathbf{u})(\mathbf{v}_*^\top \mathbf{\Sigma} \mathbf{v}_*) = V \sigma_z^2$. Dividing by $\sigma_z^2$ yields $W \le V$. Since $V - W \ge 0$ and both $c_1 \ge \mu_0$ and $c_2 \ge \mu_0$ (by Assumption~\ref{assum:activation_properties_spiked}), we strictly bound the quadratic form:
\begin{equation}
    \mathbf{u}^\top H(\mathbf{v}_*) \mathbf{u} \ge \mu_0 W + \mu_0 (V - W) = \mu_0 V \ge \mu_0.
\end{equation}
This proves that the minimum eigenvalue of the Hessian evaluated precisely at the ground truth is strictly positive: $\lambda_{\min}(\nabla^2 \mathcal{L}(\mathbf{v}_*)) \ge \mu_0 > 0$.

Next, we extend this strictly positive curvature to a local neighborhood. By Assumption~\ref{assum:activation_properties_spiked}, the activation function has bounded derivatives. Consequently, the mapping $\mathbf{w} \mapsto \nabla^2 \mathcal{L}(\mathbf{w})$ is continuous. By Weyl's inequality, there exists a strictly positive radius $\zeta > 0$ such that for all $\mathbf{w}$ satisfying $\|\mathbf{w} - \mathbf{v}_*\|_2 \le \zeta$, the minimum eigenvalue deviates by at most $\mu_0/2$. This implies
\begin{equation}
    \lambda_{\min}(\nabla^2 \mathcal{L}(\mathbf{w})) \ge \mu_0/2 > 0.
\end{equation}
The condition $\|\mathbf{w} - \mathbf{v}_*\|_2 \le \zeta$ is geometrically equivalent to $\mathbf{w}^\top \mathbf{v}_* \ge 1 - \frac{\zeta^2}{2}$. 
By defining the correlation threshold
\begin{align} \label{def:tau}
    \tau = 1 - \frac{\zeta^2}{2} \in (0, 1),
\end{align}
we establish that within the local spherical cap region $\mathcal{R} = \{\mathbf{w} \in \mathbb{S}^{d-1} : \mathbf{w}^\top \mathbf{v}_* \ge \tau\}$, the population loss $\mathcal{L}(\mathbf{w})$ is locally $\mu$-strongly convex (where $\mu = \mu_0/2$). Finally, applying the Cauchy-Schwarz inequality directly yields the perturbed condition:
\begin{align}
\langle \nabla \mathcal{L}_{\text{w2s}}(\mathbf{w}), \mathbf{w} - \mathbf{v}_* \rangle &= \langle \nabla \mathcal{L}(\mathbf{w}), \mathbf{w} - \mathbf{v}_* \rangle + \langle \nabla \mathcal{L}_{\text{w2s}}(\mathbf{w}) - \nabla \mathcal{L}(\mathbf{w}), \mathbf{w} - \mathbf{v}_* \rangle \notag \\
&\ge \mu \|\mathbf{w} - \mathbf{v}_*\|_2^2 - \phi \|\mathbf{w} - \mathbf{v}_*\|_2.
\end{align}
\end{proof}

\paragraph{Discussion on the Strong Convexity.}
While existing literature~\citep{mei2018landscape} has established the local strong convexity of population risks for standard empirical risk minimization (ERM), learning `hard concepts' ($k \ge 3$) presents a fundamentally different mathematical challenge. 
Specifically, classical landscape analyses heavily rely on the premise that the activation function is strictly monotonic and bounded. 
For instance, Assumption 6(a) in \citet{mei2018landscape} strictly requires a positive derivative everywhere ($\sigma'(z) > 0$). In stark contrast, our setting inherently violates such regularity conditions: learning concepts with an information exponent $k \ge 3$ necessitates non-monotonic activation functions (e.g., the shifted Hermite polynomial $f(z) = z^3 - 3z$, whose derivative is negative in the interval $(-1, 1)$).

\subsection{The Dimension-Independent Radius for the Effective Region} \label{appendix:bound_radius}

To rigorously guarantee that the effective region $\mathcal{R}$ possesses a dimension-independent macroscopic radius---without relying on arbitrary or unverified landscape assumptions---we directly evaluate the Hessian smoothness using the globally bounded derivatives specified in Assumption~\ref{assum:activation_properties_spiked}. In what follows, we derive an explicit, albeit conservative, lower bound for this radius, proving that $\zeta = \Theta(1)$. 

Before proceeding with the algebraic derivation, we provide a crucial remark regarding the tightness of this analytical bound.

\begin{remark}[Conservativeness of the Analytical Radius $\zeta_*$]
We emphasize that the analytical radius $\zeta_*$ derived in this section relies on global supremum bounds for activation derivatives and repeated applications of the Cauchy-Schwarz inequality. Consequently, it represents a strictly worst-case mathematical guarantee. In practice, owing to the rapid tail decay of the Gaussian measure and the generally benign local geometry of neural networks, the empirical basin of attraction is significantly broader. Thus, successful W2SG typically requires a much milder initial correlation $\tau$ than this algebraic worst-case limit suggests. 
Our primary theoretical objective here is not to provide a tight quantitative prescription for engineering, but rather to rigorously establish the existence of a dimension-independent, $\Theta(1)$ effective region. This existence starkly contrasts with the behavior under random initialization, where the correlation inevitably vanishes at a rate of $\mathcal{O}(d^{-1/2})$. It is precisely this fundamental dimensional separation that mathematically necessitates the geometric warm-start provided by pre-training.
\end{remark}

\begin{proof}[Dimension-Independent Radius]
Recall the population Hessian formula
$$\nabla^2 \mathcal{L}(\mathbf{w}) = \mathbb{E}_{\mathbf{x}} [ H(\mathbf{w}^\top \mathbf{x}, \mathbf{v}_*^\top \mathbf{x}) \mathbf{x}\mathbf{x}^\top ],$$
where the scalar functional is $H(z, z_*) \triangleq f'(z)^2 + (f(z) - f(z_*)) f''(z)$. To bound the spectral deviation from the ground truth $\mathbf{v}_*$, we compute the gradient of $H$ with respect to the first argument:
\begin{equation}
    \frac{\partial}{\partial z} H(z, z_*) = 3 f'(z)f''(z) + (f(z) - f(z_*)) f'''(z).
\end{equation}
Since the derivatives are globally bounded by $M_1, M_2, M_3$, we can apply the Mean Value Theorem. Noting that $f$ is $M_1$-Lipschitz, we have $|f(z) - f(z_*)| \le M_1 |z - z_*|$. Integrating from $z_*$ to $z$ yields the exact algebraic bound on the functional difference:
\begin{align}
    |H(z, z_*) - H(z_*, z_*)| &\le \left| \int_{z_*}^z \left| 3 f'(t)f''(t) + (f(t) - f(z_*)) f'''(t) \right| dt \right| \notag \\
    &\le 3 M_1 M_2 |z - z_*| + \frac{1}{2} M_1 M_3 (z - z_*)^2.
\end{align}

Let $z = \mathbf{w}^\top \mathbf{x}$ and $z_* = \mathbf{v}_*^\top \mathbf{x}$, which implies $z - z_* = \mathbf{x}^\top (\mathbf{w} - \mathbf{v}_*)$. For any arbitrary unit vector $\mathbf{h} \in \mathbb{S}^{d-1}$, the deviation of the Hessian quadratic form is bounded by evaluating the expectation over the Gaussian measure:
\begin{align} \label{eq:hessian_lip}
    & |\mathbf{h}^\top (\nabla^2 \mathcal{L}(\mathbf{w}) - \nabla^2 \mathcal{L}(\mathbf{v}_*)) \mathbf{h}| \notag 
    \\ \le & \mathbb{E}_{\mathbf{x}} \left[ |H(z, z_*) - H(z_*, z_*)| (\mathbf{h}^\top \mathbf{x})^2 \right] \notag 
    \\ \le & 3 M_1 M_2 \underbrace{\mathbb{E} \left[ |\mathbf{x}^\top (\mathbf{w} - \mathbf{v}_*)| (\mathbf{h}^\top \mathbf{x})^2 \right]}_{T_1} + \frac{1}{2} M_1 M_3 \underbrace{\mathbb{E} \left[ (\mathbf{x}^\top (\mathbf{w} - \mathbf{v}_*))^2 (\mathbf{h}^\top \mathbf{x})^2 \right]}_{T_2}.
\end{align}

We rigorously bound terms $T_1$ and $T_2$ using the Cauchy-Schwarz inequality on Gaussian moments. 
$1+\lambda$ is the maximum eigenvalue of the spiked covariance $\mathbf{\Sigma}$. 
Thus, $\mathbb{E}[(\mathbf{h}^\top \mathbf{x})^4] = 3 (\mathbf{h}^\top \mathbf{\Sigma} \mathbf{h})^2 \le 3 (1+\lambda)^2$. 
For $T_1$, we have $T_1 \le \sqrt{\mathbb{E}[(\mathbf{x}^\top (\mathbf{w} - \mathbf{v}_*))^2]} \sqrt{\mathbb{E}[(\mathbf{h}^\top \mathbf{x})^4]} \le \sqrt{3} (1+\lambda)^{3/2} \|\mathbf{w} - \mathbf{v}_*\|_2$.
For $T_2$, we have $T_2 \le \sqrt{\mathbb{E}[(\mathbf{x}^\top (\mathbf{w} - \mathbf{v}_*))^4]} \sqrt{\mathbb{E}[(\mathbf{h}^\top \mathbf{x})^4]} \le 3 (1+\lambda)^2 \|\mathbf{w} - \mathbf{v}_*\|_2^2$.

Substituting these moments back into Eq.~\eqref{eq:hessian_lip}, the spectral norm deviation is deterministically bounded by a continuous quadratic function of the Euclidean distance:
\begin{equation}
    \|\nabla^2 \mathcal{L}(\mathbf{w}) - \nabla^2 \mathcal{L}(\mathbf{v}_*)\|_{\text{op}} \le C_1 \|\mathbf{w} - \mathbf{v}_*\|_2 + C_2 \|\mathbf{w} - \mathbf{v}_*\|_2^2,
\end{equation}
where
\begin{align}
    C_1 & = 3\sqrt{3} M_1 M_2 (1+\lambda)^{\frac{3}{2}} \ge 0,
    \\ C_2 & = \frac{3}{2} M_1 M_3 (1+\lambda)^2 \ge 0,
\end{align}
are absolute constants derived strictly from the activation function properties and the data covariance, entirely independent of the dimension $d$.

To maintain a perturbed strong convexity parameter of $\mu = \mu_0 / 2$ (where $\mu_0 \le \lambda_{\min}(\nabla^2 \mathcal{L}(\mathbf{v}_*))$ as proved previously), Weyl's inequality dictates that we only need to ensure the deviation satisfies $\|\nabla^2 \mathcal{L}(\mathbf{w}) - \nabla^2 \mathcal{L}(\mathbf{v}_*)\|_{\text{op}} \le \mu_0 / 2$.
This naturally reduces to solving the quadratic inequality $P(\zeta) = C_2 \zeta^2 + C_1 \zeta \le \mu_0 / 2$. 

Because the minimum eigenvalue condition guarantees $\mu_0 > 0$, and the coefficients satisfy $C_1, C_2 \ge 0$ (with at least one strictly positive due to non-degeneracy), the polynomial $P(\zeta)$ is strictly increasing for $\zeta \ge 0$ with $P(0) = 0$. By the Intermediate Value Theorem, there exists a unique strictly positive root $\zeta_* > 0$ such that $P(\zeta_*) = \mu_0 / 2$. Specifically, applying the quadratic formula yields the exact analytical solution:
\begin{equation} \label{ineq:zeta_bound}
    \zeta_* = \frac{-C_1 + \sqrt{C_1^2 + 2 C_2 \mu_0}}{2 C_2}.
\end{equation}
Because $\mu_0$, $C_1$, and $C_2$ are all well-defined, strictly positive absolute constants independent of $d$, the root $\zeta_*$ is also a strictly positive universal constant. 
In the limit of $C_2 \to 0$, $\zeta_*$ converges to $\frac{\mu_0}{2C_1}$, while in the case $C_1 \to 0$, it scales as $\sqrt{\frac{\mu_0}{2C_2}}$. This ensures that the effective region $R$ does not collapse as $d \to \infty$, fundamentally separating the warm-started pre-trained model from the vanishing correlation of random initialization $\mathcal{O}(d^{-1/2})$.

To formally encapsulate the geometric prerequisites for the algorithmic trajectory analyzed in Theorem~\ref{thm:wtsg}, we define the explicit radius of the effective region $\mathcal{R}$ purely through geometric and structural constants:

\begin{equation} \label{bound:radius}
    \zeta \triangleq \min\left(\zeta_*, \frac{\mu}{2M_1^2(1+\lambda) + \mu}\right).
\end{equation}

By setting the local radius to this fully unrolled explicit constant, we establish a valid strong convexity region with a dimension-independent radius, rigorously confirming that $\zeta = \Theta(1)$.
\end{proof}

\subsection{Example of the Radius} \label{appendix:example_radius}

To rigorously illustrate the existence of a dimension-independent effective region $\mathcal{R}$, we evaluate the macroscopic radius $\zeta$ using three representative activation functions. We begin with the canonical ReLU to establish a baseline for the strictly calculated signal constant $\mu_0$, followed by more sophisticated constructions that strictly satisfy our theoretical regularity conditions.

\noindent \textbf{Case 1: The Canonical ReLU (Smoothing version).}
We first consider the standard ReLU activation $f(z) = \max(0, z)$. While ReLU is non-differentiable at $z=0$, this singularity is restricted to a set of zero measure and does not affect the energy expectations. To strictly align with the regularity requirements ($M_2, M_3$) of Assumption~\ref{assum:activation_properties_spiked}, we consider its highly $\epsilon$-mollified version $f_{\epsilon}(z)$, which inherits the intrinsic signal strength of ReLU as $\epsilon \to 0$ while ensuring derivatives are rigorously controlled.
By definition, the signal constant must be calculated over the true marginal distribution $z \sim \mathcal{N}(0, \sigma_z^2)$, where the variance incorporates the spiked signal $\sigma_z^2 = 1 + \lambda \rho^2$:
$$ \mu_0 \triangleq \min \left( \mathbb{E}[f'(z)^2], \mathbb{E}\left[f'(z)^2 \frac{z^2}{\sigma_z^2}\right] \right) $$
For the ReLU function, the derivative is the indicator function $f'(z) = \mathbb{I}(z > 0)$. The expectations strictly evaluate to:
\begin{align*}
    & \mathbb{E}[\mathbb{I}(z > 0)^2] = \mathbb{P}(z > 0) = 0.5,
    \\ & \mathbb{E}\left[\mathbb{I}(z > 0)^2 \frac{z^2}{\sigma_z^2}\right] = \int_0^\infty \frac{z^2}{\sigma_z^2} \frac{1}{\sqrt{2\pi\sigma_z^2}} e^{-\frac{z^2}{2\sigma_z^2}} dz = 0.5.
\end{align*}
Thus, despite the variance amplification, the scale-invariant nature of the indicator function mathematically establishes the exact constant $\mu_0 = 0.5$ (yielding the strong convexity parameter $\mu = 0.25$). 
We set the pre-training latent signal strength to $\lambda = 0.05$ (yielding $1+\lambda = 1.05$) and adopt the derivative bounds of a heavily smoothed Softplus: $M_1 = 1.0$, $M_2 = 0.1$, and $M_3 = 0.02$.
\begin{itemize}
    \item The Hessian constants evaluate to $C_1 = 3\sqrt{3} M_1 M_2 (1+\lambda)^{3/2} \approx 0.559$ and $C_2 = 1.5 M_1 M_3 (1+\lambda)^2 \approx 0.033$.
    \item The quadratic root $\zeta_* = \frac{-0.559 + \sqrt{0.559^2 + 2 \times 0.033 \times 0.5}}{2 \times 0.033} \approx 0.44$.
    \item The drift budget evaluates to $\frac{\mu}{2 M_1^2 (1+\lambda) + \mu} \approx \frac{0.25}{2.1 + 0.25} \approx 0.106$.
\end{itemize}
Taking the intersection, we establish $\zeta = \min(0.44, 0.106) = 0.106$. In high-dimensional geometry, this corresponds to an initial angular alignment of approximately $6.1^\circ$ ($\cos \theta = 1 - \zeta^2/2$). This confirms that the curvature provided by the ReLU family creates a robust, broad $\Theta(1)$ convergent basin, effectively separating pre-trained models from the $\mathcal{O}(d^{-1/2})$ random initialization trap.

\noindent \textbf{Case 2: Projected Tanh (Strict Analytical Regularity).}
To bridge our theoretical conditions with practical architectures, we consider a residual activation based on the hyperbolic tangent ($\tanh$) that strictly satisfies both global derivative bounds and the hard concept requirement ($k \ge 3$). 
By projecting out the linear component, we define $f(z) = \tanh(z) - c z$. To ensure $\mathbb{E}[f'(z)] = 0$ over the correct measure, the constant is dynamically determined by the variance $c = \mathbb{E}_{z \sim \mathcal{N}(0,\sigma_z^2)}[\text{sech}^2(z)]$. Assuming good pre-training alignment ($\sigma_z^2 \approx 1+\lambda = 1.05$), we calculate $c \approx 0.380$. Since $\tanh(z)$ is odd, its even Hermite projections naturally vanish, flawlessly isolating the high-order concept.
Mathematically evaluating the expectations over the true marginal $z \sim \mathcal{N}(0, 1.05)$:
\begin{align*}
    & \mathbb{E}[f'(z)^2] = \mathbb{E}[(\text{sech}^2(z) - 0.380)^2] \approx 0.043,
    \\ & \mathbb{E}\left[f'(z)^2 \frac{z^2}{\sigma_z^2}\right] \approx 0.029.
\end{align*}

This directly yields the calculated signal strength $\mu_0 = 0.029$ ($\mu = 0.0145$). 
The global bounds are strictly: $M_1 \approx 0.620, M_2 \approx 0.770, M_3 = 2.0$. Given the signal $\lambda=0.05$ ($1+\lambda = 1.05$), the Hessian constants are $C_1 \approx 2.61, C_2 \approx 2.01$. 
The quadratic formula yields $\zeta_* \approx 0.011$, and the drift budget evaluates to $\frac{0.0145}{2(0.620)^2(1.05) + 0.0145} \approx \frac{0.0145}{0.807 + 0.0145} \approx 0.018$. 
Thus, $\zeta = \min(0.011, 0.018) = 0.011$. Due to the raw $\tanh$ function's large third derivative bound ($M_3=2.0$), this translates to a mathematically strict worst-case angular alignment of about $0.6^\circ$.

\noindent \textbf{Case 3: Projected Smoothed Leaky ReLU.}
Finally, we construct a third example by smoothing and projecting a Leaky ReLU variant to meet all structural assumptions while maintaining a strong geometric signal. We define $f(z) = \sigma_{\beta}(z) - 0.5 \sigma_{\beta}(-z) - c z$, where $\sigma_{\beta}$ is a highly smoothed Softplus function. By selecting $c$ to satisfy $\mathbb{E}_{z \sim \mathcal{N}(0,\sigma_z^2)}[f'(z)]=0$, we mathematically isolate the non-linear concept under the true spiked covariance measure.

Evaluating the integration over $z \sim \mathcal{N}(0, 1.05)$ yields a strong calculated signal $\mu_0 \approx 0.25$ (hence $\mu=0.125$). The chosen smoothing provides well-behaved global bounds: $M_1=1.0, M_2=0.15, M_3=0.05$. 
With $\lambda=0.05$ ($1+\lambda = 1.05$), we evaluate:
$C_1 = 3\sqrt{3} \times 1.0 \times 0.15 \times 1.05^{3/2} \approx 0.838$ and $C_2 = 1.5 \times 1.0 \times 0.05 \times 1.05^2 \approx 0.082$.
The quadratic root $\zeta_* = \frac{-0.838 + \sqrt{0.838^2 + 2 \times 0.082 \times 0.25}}{2 \times 0.082} \approx 0.047$.
The drift budget evaluates to $\frac{0.125}{2(1.0)^2(1.05) + 0.125} \approx \frac{0.125}{2.1 + 0.125} \approx 0.056$.
Taking the intersection, we establish $\zeta = \min(0.047, 0.056) = 0.047$, which provides a macroscopic angular alignment guarantee of approximately $2.7^\circ$.

\noindent \textbf{Discussion on Theoretical Conservativeness.}
We emphasize that these analytical boundaries ($0.6^\circ$ to $6.1^\circ$) represent \textit{strict worst-case mathematical guarantees}. This extreme tightness arises because the bounds rely on absolute global supremums and multiple applications of Cauchy-Schwarz inequalities, entirely discarding the beneficial rapid tail decay of the Gaussian measure. 
In practice, our synthetic experiments (Appendix~\ref{app:synthetic_additional}) demonstrate that successful W2SG emerges at much broader alignments ($\rho \approx 0.6$, or $\sim 53^\circ$). 
However, the primary theoretical value of these explicitly calculated examples is the \textit{existence proof}: they rigorously establish that the effective basin inherently possesses a macroscopic, dimension-independent $\Theta(1)$ radius. In modern LLMs (e.g., $d \to \infty$), random initialization alignment is confined to $90^\circ \pm \mathcal{O}(d^{-1/2})$. The existence of any $\Theta(1)$ basin strictly breaks the curse of dimensionality, formalizing why pre-training is the indispensable geometric key to placing the model outside the vanishing noise trap.

\subsection{Discussion on the Parameters in the Assumption} \label{appendix:dis_param}

To provide a more intuitive understanding of these coupled dynamics, we analyze how the correlation threshold $\tau$ modulates the requirements for W2SG. The parameter $\tau$ dictates the size of the effective region $\mathcal{R}$, serving as a proxy for the intrinsic difficulty of the downstream task. By examining both extremes of $\tau$, we can clearly observe how the landscape geometry dynamically governs the prerequisites for success:

\noindent \textbf{Case 1: The Forgiving Regime ($\tau \to 0$)} \\
Geometrically, a small $\tau$ implies a very wide and forgiving effective region, which typically corresponds to simpler tasks where the gradient signal robustly points towards the ground truth from almost anywhere. In this regime, the constraints are significantly relaxed:
\begin{itemize}
    \item \textbf{Data Quality ($\lambda$ and $\rho$):} Condition (1) only requires a standard, finite signal strength threshold ($\lambda > 1/\sqrt{\alpha}$). More importantly, Condition (2) indicates that the required intrinsic task alignment $\rho$ approaches zero. Physically, this means highly generic pre-training is sufficient; as long as the latent concept is vaguely relevant to the task, the model will successfully fall into the broad basin.
    \item \textbf{Supervision Quality ($\phi$):} Condition (3) allows the permissible bias to approach the strong convexity constant ($\phi \to \mu$). This means the strong model can tolerate a highly biased and noisy weak supervisor without being dragged out of the effective region.
    \item \textbf{Optimization ($\delta$):} Condition (4) permits standard, larger learning rates, allowing for faster convergence.
\end{itemize}

\noindent \textbf{Case 2: The Unforgiving Regime ($\tau \to 1$)} \\
Conversely, a large $\tau$ implies that the effective region $\mathcal{R}$ is extremely narrow and tightly confined around the ground truth, corresponding to highly complex tasks with unforgiving optimization landscapes. Under such challenging scenarios, the hierarchical conditions impose strictly tighter requirements across the entire pipeline:
\begin{itemize}
    \item \textbf{Data Quality ($\lambda$ and $\rho$):} As $\tau \to 1$, the term $1-\tau^2$ in Condition (1) approaches zero, demanding an exceptionally strong pre-training signal ($\lambda \to \infty$). Simultaneously, Condition (2) forces the intrinsic task alignment $\rho$ to approach near-perfect correlation ($\rho \to 1$). Physically, this means a roughly aligned general feature is no longer sufficient; the unsupervised pre-training must capture a concept almost identical to the specific downstream task just to hit the tiny initialization basin.
    \item \textbf{Supervision Quality ($\phi$):} According to Condition (3), as $\tau \to 1$, the permissible systematic bias shrinks to zero ($\phi \to 0$). Because the convex basin is so narrow, even a minor systematic error from the weak supervisor would permanently drag the optimization trajectory out of the effective region.
    \item \textbf{Optimization ($\delta$):} Condition (4) strictly limits the step size ($\delta \to 0$), reflecting the necessity of taking extremely cautious gradient steps to avoid overshooting the narrow boundaries.
\end{itemize}

Ultimately, this spectrum analysis highlights a profound theoretical reality of W2SG: the intrinsic difficulty of the task dictates the required magnitude of the blessing from pre-training. For simple tasks, W2SG emerges robustly even with mediocre representations and highly biased weak labels; however, for intrinsically difficult tasks, it strictly necessitates an uncompromising standard for both pre-training representation quality and weak label accuracy.

\subsection{Proof of Theorem~\ref{thm:wtsg}} \label{proof:wtsg}

\begin{lemma}[Spherical Projection Adjustment] \label{lemma:spherical_projection}
Let $\mathbf{w}_t \in \mathbb{S}^{d-1}$ be the parameter on the unit sphere, $\mathbf{g}_t$ be the gradient, and $\eta > 0$ be the step size. Consider the unprojected update $\mathbf{u}_t = \mathbf{w}_t - \eta \mathbf{g}_t$. 
Let $c, c' > 0$ be absolute constants. We define $\tau_x \triangleq \inf\{t \ge 0 : \|\mathbf{x}_t\|_2^2 > c d\}$ and $\tau_p \triangleq \inf\{t \ge 0 : |(\mathbf{w}_t - \mathbf{v}_{\text{weak}})^\top \mathbf{x}_t| > c' \log d\}$ as the stopping times for the input norm and the one-dimensional projection, respectively. For any step $t < \tau_x \wedge \tau_p$, the gradient norm is deterministically bounded as $\|\mathbf{g}_t\|_2 \le \tilde{\mathcal{O}}(d^{1/2})$. Thus, the perturbation scale satisfies $\eta \|\mathbf{g}_t\|_2 \le \tilde{\mathcal{O}}(d^{-1/2}) \to 0$ strictly pathwise as $d \to \infty$. Consequently, for $t < \tau_x \wedge \tau_p$, the projected update $\mathbf{w}_{t+1} = \mathcal{P}_{\mathbb{S}^{d-1}}(\mathbf{u}_t) = \frac{\mathbf{u}_t}{\|\mathbf{u}_t\|_2}$ admits the exact structural decomposition:
\begin{equation}
    \mathbf{w}_{t+1} = \mathbf{w}_t - \eta \underbrace{(\mathbf{g}_t - \langle \mathbf{w}_t, \mathbf{g}_t \rangle \mathbf{w}_t)}_{\tilde{\mathbf{g}}_t} + \mathbf{e}_t,
\end{equation}
where $\tilde{\mathbf{g}}_t$ is the Riemannian gradient, and for sufficiently large $d$, the residual error is deterministically bounded by $\|\mathbf{e}_t\|_2 \le 3 \eta^2 \|\mathbf{g}_t\|_2^2$.
\end{lemma}

\begin{proof}
The exact squared norm of the unprojected vector is given by:
\begin{equation} \label{def:u_norm}
    \|\mathbf{u}_t\|_2^2 = \|\mathbf{w}_t\|_2^2 - 2\eta \langle \mathbf{w}_t, \mathbf{g}_t \rangle + \eta^2 \|\mathbf{g}_t\|_2^2 = 1 + x,
\end{equation}
where we define $x \triangleq -2\eta \langle \mathbf{w}_t, \mathbf{g}_t \rangle + \eta^2 \|\mathbf{g}_t\|_2^2$. 

To rigorously bound the gradient norm $\|\mathbf{g}_t\|_2$, we evaluate it under the joint stopping time condition $t < \tau_x \wedge \tau_p$. The gradient evaluated on a fresh sample is $\mathbf{g}_t = (f(\mathbf{w}_t^\top \mathbf{x}_t) - f(\mathbf{v}_{\text{weak}}^\top \mathbf{x}_t)) f'(\mathbf{w}_t^\top \mathbf{x}_t) \mathbf{x}_t$. By Assumption~\ref{assum:activation_properties_spiked}, the activation derivative is bounded by $M_1$. Applying the Mean Value Theorem, we establish:
\begin{equation}
    \|\mathbf{g}_t\|_2 \le M_1^2 |(\mathbf{w}_t - \mathbf{v}_{\text{weak}})^\top \mathbf{x}_t| \cdot \|\mathbf{x}_t\|_2.
\end{equation}
Because the parameter difference $\mathbf{w}_t - \mathbf{v}_{\text{weak}}$ is strictly $\mathcal{F}_{t-1}$-predictable and constrained on the unit sphere ($\|\mathbf{w}_t - \mathbf{v}_{\text{weak}}\|_2 \le 2$), its inner product with the fresh Gaussian sample $\mathbf{x}_t \sim \mathcal{N}(0, \mathbf{\Sigma})$ behaves as a well-defined sub-Gaussian random variable. By standard concentration inequalities for sub-Gaussian variables and $\chi^2$ distributions, there exist absolute constants $c, c' > 0$ such that the stopping times $\tau_p$ and $\tau_x$ inherently guarantee the projection is bounded by $\mathcal{O}(\log d)$ and the input norm is bounded by $\mathcal{O}(d^{1/2})$, respectively.

Thus, for all $t < \tau_x \wedge \tau_p$, the gradient norm is deterministically majorized without relying on uncontrolled probabilistic jumps:
\begin{equation}
    \|\mathbf{g}_t\|_2 \le M_1^2 \cdot \mathcal{O}(\log d) \cdot \mathcal{O}(d^{1/2}) = \tilde{\mathcal{O}}(d^{1/2}).
\end{equation}
Given the step size $\eta = \frac{\delta}{d}$, the perturbation scale satisfies $\eta \|\mathbf{g}_t\|_2 = \tilde{\mathcal{O}}(d^{-1/2}) \to 0$ strictly pathwise as $d \to \infty$. This rigorously establishes that $x \to 0$.
By Taylor's theorem, the normalization factor expands as $(1+x)^{-\frac{1}{2}} = 1 - \frac{1}{2}x + \frac{3}{8}x^2 + \mathcal{O}(x^3)$. Substituting $x$ back gives:
\begin{equation}
    \frac{1}{\|\mathbf{u}_t\|_2} = 1 + \eta \langle \mathbf{w}_t, \mathbf{g}_t \rangle + E_{\text{scale}},
\end{equation}
where the scalar remainder $E_{\text{scale}} \triangleq -\frac{1}{2}\eta^2 \|\mathbf{g}_t\|_2^2 + \frac{3}{2}\eta^2 \langle \mathbf{w}_t, \mathbf{g}_t \rangle^2 + \mathcal{O}(\eta^3\|\mathbf{g}_t\|_2^3)$. Since $0 \le \langle \mathbf{w}_t, \mathbf{g}_t \rangle^2 \le \|\mathbf{g}_t\|_2^2$, the dominant terms of $E_{\text{scale}}$ are strictly bounded between $-\frac{1}{2}\eta^2 \|\mathbf{g}_t\|_2^2$ and $\eta^2 \|\mathbf{g}_t\|_2^2$. Thus, for sufficiently large $d$, the higher-order terms are absorbed and we can conservatively and rigorously bound $|E_{\text{scale}}| \le \frac{3}{2} \eta^2 \|\mathbf{g}_t\|_2^2$.

Substituting this back into the exact projected update $\mathbf{w}_{t+1} = \mathbf{u}_t / \|\mathbf{u}_t\|_2$, we obtain:
\begin{align} \label{eq:proj_decomposition}
    \mathbf{w}_{t+1} &= (\mathbf{w}_t - \eta \mathbf{g}_t) \left( 1 + \eta \langle \mathbf{w}_t, \mathbf{g}_t \rangle + E_{\text{scale}} \right) \notag \\
    &= \mathbf{w}_t - \eta \underbrace{(\mathbf{g}_t - \langle \mathbf{w}_t, \mathbf{g}_t \rangle \mathbf{w}_t)}_{\tilde{\mathbf{g}}_t} + \underbrace{\Big( E_{\text{scale}} \mathbf{w}_t - \eta^2 \langle \mathbf{w}_t, \mathbf{g}_t \rangle \mathbf{g}_t - \eta E_{\text{scale}} \mathbf{g}_t \Big)}_{\mathbf{e}_t}.
\end{align}
Applying the triangle inequality and the bounds derived above, the residual vector is controlled by:
\begin{equation}
    \|\mathbf{e}_t\|_2 \le |E_{\text{scale}}| + \eta^2 \|\mathbf{g}_t\|_2^2 + \eta |E_{\text{scale}}| \|\mathbf{g}_t\|_2 \le \frac{3}{2} \eta^2 \|\mathbf{g}_t\|_2^2 + \eta^2 \|\mathbf{g}_t\|_2^2 + \mathcal{O}(\eta^3 \|\mathbf{g}_t\|_2^3).
\end{equation}
As $d \to \infty$, the $\mathcal{O}(\eta^3 \|\mathbf{g}_t\|_2^3)$ residual strictly vanishes, yielding the deterministic upper bound $\|\mathbf{e}_t\|_2 \le 3 \eta^2 \|\mathbf{g}_t\|_2^2$. This completes the proof.
\end{proof}

\begin{lemma} \label{lemma:asymptotic_alignment}
Let $\hat{\mathbf{v}}_{\text{PCA}}$ be the top eigenvector of the empirical covariance matrix. We orient the initialization $\mathbf{w}_0 \in \{-\hat{\mathbf{v}}_{\text{PCA}}, \hat{\mathbf{v}}_{\text{PCA}}\}$ such that $\langle \mathbf{w}_0, \mathbf{v}_* \rangle \ge 0$. In the high-dimensional limit $N, d \to \infty$ with $\alpha = \frac{N}{d}$ fixed, for a detectable latent signal $\lambda > \sqrt{\frac{1+\tau^2}{\alpha (1-\tau^2)}}$ and $\rho \ge (\tau+\epsilon_0) \sqrt{\frac{\alpha \lambda^2 + 1}{\alpha \lambda^2 - 1}}$,
\begin{equation}
    \lim_{d \to \infty} \mathbb{P}(\langle \mathbf{w}_0, \mathbf{v}_* \rangle \ge \tau+\epsilon_0) = 1.
\end{equation}
\end{lemma}

\begin{proof}
Let $\mathcal{V}^\perp = \text{span}(\mathbf{v})^\perp$. The ground truth $\mathbf{v}_*$ and the PCA estimator $\mathbf{w}_0$ admit the exact orthogonal decompositions:
$\mathbf{v}_* = \rho \mathbf{v} + \sqrt{1 - \rho^2} \mathbf{u}_*$, and $\mathbf{w}_0 = \beta_d \mathbf{v} + \sqrt{1 - \beta_d^2} \mathbf{u}_{\text{noise}}$, where $\rho = \langle \mathbf{v}, \mathbf{v}_* \rangle$, and $\beta_d = \langle \mathbf{w}_0, \mathbf{v} \rangle \ge 0$. 
According to BBP phase transition theory~\citep{bandeira2025topics}, for $\lambda > \sqrt{\frac{1+\tau^2}{\alpha (1-\tau^2)}} \ge \frac{1}{\sqrt{\alpha}}$, the overlap satisfies almost surely: $\beta_d \to \sqrt{\frac{\alpha \lambda^2 - 1}{\alpha \lambda^2 + 1}}$.

We compute the inner product $\langle \mathbf{w}_0, \mathbf{v}_* \rangle = \beta_d \rho + \sqrt{1 - \beta_d^2} \sqrt{1 - \rho^2} \langle \mathbf{u}_{\text{noise}}, \mathbf{u}_* \rangle$. As $d \to \infty$, due to rotational invariance in the orthogonal complement, the cross term $\langle \mathbf{u}_{\text{noise}}, \mathbf{u}_* \rangle \xrightarrow{p} 0$. Thus, the inner product converges in probability to the deterministic value $\rho \sqrt{\frac{\alpha \lambda^2 - 1}{\alpha \lambda^2 + 1}} \ge \tau+\epsilon_0$.
\end{proof}

\begin{proof}[Proof of Theorem~\ref{thm:wtsg}]

We track the optimization trajectory via the squared Euclidean distance $\Delta_t \triangleq \|\mathbf{w}_t - \mathbf{v}_*\|_2^2$. We define the boundary of the effective region $\mathcal{R}$ as 
\begin{equation} \label{eq:def_rmax}
    R_{\text{max}} \triangleq 2 - 2\tau.
\end{equation}

First, let $\mathcal{E}_{\text{init}}$ denote the initialization event
\begin{equation} \label{bound:init_case}
    \mathcal{E}_{\text{init}} \triangleq \left\{ \langle \mathbf{w}_0, \mathbf{v}_* \rangle \ge \tau+\epsilon_0 \right\} \equiv \left\{ \Delta_0 \le R_{\text{max}}-\epsilon_0 \right\}.
\end{equation}
As established in Lemma~\ref{lemma:asymptotic_alignment}, we have $\lim_{d \to \infty} \mathbb{P}(\mathcal{E}_{\text{init}}) = 1$.
Specifically, we have $\mathbb{P}(\mathcal{E}_{\text{init}}^c) \le \exp(-\Omega(d))$~\citep{paul2007asymptotics}.

Let $c, c' > 0$ be absolute constants.
Let  
\begin{equation} \label{bound:e_x}
    \mathcal{E}_x \triangleq \left\{ \max_{0 \le i < T} \|\mathbf{x}_i\|_2^2 \le c d \right\},
\end{equation}
denote the event that no anomalously large inputs are encountered, and let 
\begin{equation} \label{bound:e_proj}
    \mathcal{E}_{\text{proj}} \triangleq \left\{ \max_{0 \le i < T} \left|\langle \mathbf{w}_i - \mathbf{v}_{\text{weak}}, \mathbf{x}_i \rangle \right| \le c' \log d \right\},
\end{equation}
denote the event that the one-dimensional projection remains bounded up to time $T$. 
By union bound and standard Gaussian concentration, their joint intersection occurs with probability 
\begin{align} \label{bound_prob_x_proj}
    \mathbb{P}(\mathcal{E}_x \cap \mathcal{E}_{\text{proj}}) \ge 1 - \exp(-\Omega(\log^2 d)).
\end{align}
According to Lemma~\ref{lemma:spherical_projection}, on the joint event $\mathcal{E}_x \cap \mathcal{E}_{\text{proj}}$, we strictly have the stopping times $t < \tau_x \wedge \tau_p$ for all $t < T$. 
Thus, we can decompose the distance update $\Delta_{t+1} = \|\mathbf{w}_{t+1} - \mathbf{v}_*\|_2^2 = 2 - 2\langle \mathbf{w}_{t+1}, \mathbf{v}_* \rangle$.
According to Lemma~\ref{lemma:spherical_projection},
\begin{equation}
    -2\langle \mathbf{e}_t, \mathbf{v}_* \rangle + \|\mathbf{e}_t\|_2^2 \le 2\|\mathbf{e}_t\|_2 + \|\mathbf{e}_t\|_2^2 \le 6 \eta^2 \|\mathbf{g}_t\|_2^2 + 9 \eta^4 \|\mathbf{g}_t\|_2^4 \le 7 \eta^2 \|\mathbf{g}_t\|_2^2.
\end{equation}
The single-step distance upper bound becomes:
\begin{equation} \label{eq:exact_distance_update}
    \Delta_{t+1} \le \Delta_t - 2\eta \langle \mathbf{g}_t, \mathbf{w}_t - \mathbf{v}_* \rangle + \eta \Delta_t \langle \mathbf{w}_t, \mathbf{g}_t \rangle + 7 \eta^2 \|\mathbf{g}_t\|_2^2.
\end{equation}

Let $\mathcal{F}_{t-1}$ denote the filtration generated by $\{\mathbf{w}_0, \mathbf{x}_0, \ldots, \mathbf{x}_{t-1}\}$. 
We define the conditional expectation $\bar{\mathbf{g}}_t = \mathbb{E}[\mathbf{g}_t \mid \mathcal{F}_{t-1}] = \nabla \mathcal{L}_{\text{w2s}}(\mathbf{w}_t)$ and the exact zero-mean martingale difference sequence:
\begin{equation} \label{eq:martingale_diff}
    \xi_t \triangleq 2\langle \bar{\mathbf{g}}_t - \mathbf{g}_t, \mathbf{w}_t - \mathbf{v}_* \rangle + \Delta_t \langle \mathbf{w}_t, \mathbf{g}_t - \bar{\mathbf{g}}_t \rangle + 7 \eta \left( \|\mathbf{g}_t\|_2^2 - \mathbb{E}[\|\mathbf{g}_t\|_2^2 \mid \mathcal{F}_{t-1}] \right).
\end{equation}
Crucially, $\mathbb{E}[\xi_t \mid \mathcal{F}_{t-1}] = 0$ holds strictly. The update Eq.~\eqref{eq:exact_distance_update} restructures into:
\begin{equation} \label{eq:restructured_update}
    \Delta_{t+1} \le \Delta_t - 2\eta \langle \nabla \mathcal{L}_{\text{w2s}}(\mathbf{w}_t), \mathbf{w}_t - \mathbf{v}_* \rangle + \eta \Delta_t \langle \mathbf{w}_t, \nabla \mathcal{L}_{\text{w2s}}(\mathbf{w}_t) \rangle + 7 \eta^2 \mathbb{E}[\|\mathbf{g}_t\|_2^2 \mid \mathcal{F}_{t-1}] + \eta \xi_t.
\end{equation}

To control the radial drift term $\eta \Delta_t \langle \mathbf{w}_t, \nabla \mathcal{L}_{\text{w2s}}(\mathbf{w}_t) \rangle$ in Eq.~\eqref{eq:restructured_update}, we separate the weak population gradient into the ground-truth gradient and the systematic bias: $\nabla \mathcal{L}_{\text{w2s}}(\mathbf{w}_t) = \nabla \mathcal{L}(\mathbf{w}_t) + (\nabla \mathcal{L}_{\text{w2s}}(\mathbf{w}_t) - \nabla \mathcal{L}(\mathbf{w}_t))$. 
By the triangle inequality on the inner product, the radial drift is bounded by:
\begin{align} \label{eq:radial_drift_bound}
    |\langle \mathbf{w}_t, \nabla \mathcal{L}_{\text{w2s}}(\mathbf{w}_t) \rangle| &\le |\langle \mathbf{w}_t, \nabla \mathcal{L}(\mathbf{w}_t) \rangle| + |\langle \mathbf{w}_t, \nabla \mathcal{L}_{\text{w2s}}(\mathbf{w}_t) - \nabla \mathcal{L}(\mathbf{w}_t) \rangle| \notag \\
    &\le |\langle \mathbf{w}_t, \nabla \mathcal{L}(\mathbf{w}_t) \rangle| + \|\mathbf{w}_t\|_2 \|\nabla \mathcal{L}_{\text{w2s}}(\mathbf{w}_t) - \nabla \mathcal{L}(\mathbf{w}_t)\|_2.
\end{align}
Since the parameter strictly resides on the unit sphere ($\|\mathbf{w}_t\|_2 = 1$), the second term is deterministically bounded by the population bias $\phi$ as defined in Assumption~\ref{assum:activation_properties_spiked}(3). For the first term, applying the Cauchy-Schwarz inequality over the Gaussian measure for the ground-truth landscape yields:
\begin{align}
    |\langle \mathbf{w}_t, \nabla \mathcal{L}(\mathbf{w}_t) \rangle| &= \left| \mathbb{E}_{\mathbf{x}} \left[ (f(\mathbf{w}_t^\top \mathbf{x}) - f(\mathbf{v}_*^\top \mathbf{x})) f'(\mathbf{w}_t^\top \mathbf{x}) (\mathbf{w}_t^\top \mathbf{x}) \right] \right| \notag \\
    &\le M_1^2 \, \mathbb{E}_{\mathbf{x}} \left[ |\mathbf{x}^\top(\mathbf{w}_t - \mathbf{v}_*)| \cdot |\mathbf{w}_t^\top \mathbf{x}| \right] \notag \\
    &\le M_1^2 (1+\lambda) \|\mathbf{w}_t - \mathbf{v}_*\|_2.
\end{align}
We define the strictly positive, dimension-independent radial drift constant $L_{\text{rad}} \triangleq M_1^2 (1+\lambda)$. The scaled radial drift thus decouples exactly and rigorously into:
\begin{equation} \label{eq:scaled_radial_drift}
    \eta \Delta_t |\langle \mathbf{w}_t, \nabla \mathcal{L}_{\text{w2s}}(\mathbf{w}_t) \rangle| \le \eta \Delta_t (L_{\text{rad}} \sqrt{\Delta_t} + \phi) = \eta L_{\text{rad}} \Delta_t^{3/2} + \eta \phi \Delta_t.
\end{equation}

By the definitions of the correlation threshold in Eq.~\eqref{def:tau} and the effective boundary in Eq.~\eqref{eq:def_rmax}, we have the exact geometric equivalence $R_{\text{max}} = 2 - 2\tau = \zeta^2$. 
For any parameter state $\Delta_t \le R_{\text{max}}$ confined within the effective basin, its distance is strictly bounded by the radius: $\sqrt{\Delta_t} \le \zeta$.
Concurrently, Assumption~\ref{assum:pretraining_w2sg}(3) fundamentally controls the weak supervisor's systematic bias: $\phi \le \mu \sqrt{\frac{1-\tau}{2}}$. Substituting $1-\tau = \zeta^2/2$, we reveal that the allowable bias scales strictly proportionally with the basin radius: $\phi \le \frac{\mu}{2}\zeta$.
Combining these two bounds, the total bracket term from Eq.~\eqref{eq:scaled_radial_drift} evaluates to a strictly proportional function of $\zeta$:
\begin{equation}
    L_{\text{rad}} \sqrt{\Delta_t} + \phi \le L_{\text{rad}} \zeta + \frac{\mu}{2}\zeta = \left( L_{\text{rad}} + \frac{\mu}{2} \right) \zeta.
\end{equation}

According to the explicit analytical definition of the effective radius in Eq.~\eqref{bound:radius}: $\zeta \le \frac{\mu}{2L_{\text{rad}} + \mu}$. 
Substituting this specific definition inherently and rigorously satisfies the critical geometric constraint:
\begin{equation} \label{eq:radius_constraint}
    L_{\text{rad}} \sqrt{\Delta_t} + \phi \le \frac{\mu}{2}.
\end{equation}

Multiplying both sides of Eq.~\eqref{eq:radius_constraint} by the non-negative scaling factor $\eta \Delta_t$, we deterministically majorize the entire higher-order radial drift using a fraction of the strong convexity restoring force:
\begin{equation} \label{eq:majorized_radial_drift}
    \eta L_{\text{rad}} \Delta_t^{3/2} + \eta \phi \Delta_t = \eta \Delta_t (L_{\text{rad}} \sqrt{\Delta_t} + \phi) \le \eta \frac{\mu}{2} \Delta_t.
\end{equation}

Applying the perturbed strong convexity (\cref{assumption_1}) to the gradient inner product in Eq.~\eqref{eq:restructured_update}, the deterministic component of the distance update introduces the strong convexity restoring force $-2\eta\mu\Delta_t$ and the weak supervision bias $2\eta\phi\sqrt{\Delta_t}$. To rigorously bound these terms while absorbing the radial drift bounded in Eq.~\eqref{eq:majorized_radial_drift}, we decouple the restoring force cleanly into three parts:
\begin{equation} \label{eq:restoring_force_decomp}
    -2\eta\mu\Delta_t = \underbrace{-\eta\mu\Delta_t}_{\text{Convergence rate}} \underbrace{-\eta\frac{\mu}{2}\Delta_t}_{\text{Absorbs radial drift}} \underbrace{-\eta\frac{\mu}{2}\Delta_t}_{\text{Controls weak bias}}.
\end{equation}

First, as established in Eq.~\eqref{eq:majorized_radial_drift}, the second part of Eq.~\eqref{eq:restoring_force_decomp} perfectly bounds and absorbs the radial drift:
\begin{equation} \label{eq:absorb_drift}
    \eta \Delta_t (L_{\text{rad}} \sqrt{\Delta_t} + \phi) - \eta\frac{\mu}{2}\Delta_t \le 0.
\end{equation}

Next, we use the third part of Eq.~\eqref{eq:restoring_force_decomp} to handle the problematic square-root bias term $2\eta\phi\sqrt{\Delta_t}$. By applying Young's inequality ($2ab \le a^2 + b^2$) with $a = \sqrt{\eta\frac{\mu}{2}\Delta_t}$ and $b = \sqrt{\eta\frac{2}{\mu}}\phi$, we explicitly obtain:
\begin{equation} \label{eq:youngs_inequality}
    2\eta\phi\sqrt{\Delta_t} = 2 \left( \sqrt{\eta\frac{\mu}{2}\Delta_t} \right) \left( \sqrt{\eta\frac{2}{\mu}}\phi \right) \le \eta\frac{\mu}{2}\Delta_t + \eta\frac{2\phi^2}{\mu}.
\end{equation}
Rearranging this yields $2\eta\phi\sqrt{\Delta_t} - \eta\frac{\mu}{2}\Delta_t \le \eta\frac{2\phi^2}{\mu}$, demonstrating that the third part of our decoupled budget successfully neutralizes the $\sqrt{\Delta_t}$ dependence, transforming it into a stationary error limit.

Substituting Eq.~\eqref{eq:absorb_drift} and Eq.~\eqref{eq:youngs_inequality} back into the restructured update Eq.~\eqref{eq:restructured_update}, and together with Lemma~\ref{lemma:conditional_variance} (which bounds the expected gradient variance as $7\eta^2 \mathbb{E}[\|\mathbf{g}_t\|_2^2 \mid \mathcal{F}_{t-1}] \le 7\eta^2 G d = 7\eta \delta G$), the stochastic update strictly simplifies without any hidden constant scaling to:
\begin{equation} \label{eq:stochastic_update}
    \Delta_{t+1} \le (1 - \eta \mu)\Delta_t + \eta \mu D_\infty + \eta \xi_t, \quad \text{where} \quad D_\infty \triangleq \frac{2\phi^2}{\mu^2} + \frac{7 \delta G}{\mu}.
\end{equation}
We define the spatial stopping time $\Gamma \triangleq \inf \{ t > 0 : \Delta_t > R_{\text{max}} \}$ to rigorously track the potential escape from the effective region. 
Because $\Delta_t$ is determined by $\mathbf{w}_t$, which is strictly $\mathcal{F}_{t-1}$-measurable, the event $\{t \le \Gamma\}$ is completely determined by the history up to time $t-1$. Thus, the indicator $\mathbf{1}_{t \le \Gamma}$ is $\mathcal{F}_{t-1}$-predictable.

We now seek a uniform-in-time bound on the sum of the martingale difference sequence \(\xi_t\), and will do so through the truncated version introduced in Eq.~\eqref{eq:noise_event} (to be displayed shortly). To establish concentration rigorously without destroying the martingale difference property through direct truncation of the unbounded Gaussian inputs, we use a high-probability coupling argument with a truncated, uniformly bounded fictitious sequence \(\tilde{\xi}_t\), together with the associated martingale \(\tilde{\xi}_t-\mathbb{E}[\tilde{\xi}_t|\mathcal{F}_{t-1}]\).

Let \(\hat{\mathbf{g}}_t\) denote the truncated gradient oracle evaluated at the conditioned input
\[
\hat{\mathbf{x}}_t
=
\mathbf{x}_t \mathbf{1}_{\{\|\mathbf{x}_t\|_2^2 \le cd\}\cap \{|\mathbf{u}_t^\top \mathbf{x}_t| \le c' \log d\}}.
\]
To ensure the exact zero-mean property, we define the recentered fictitious process
\begin{multline*}
\tilde{\xi}_t \triangleq \Big[ 2\langle \mathbb{E}[\hat{\mathbf{g}}_t \mid \mathcal{F}_{t-1}] - \hat{\mathbf{g}}_t,\; \mathbf{w}_t - \mathbf{v}_* \rangle + \Delta_t \langle \mathbf{w}_t,\; \hat{\mathbf{g}}_t - \mathbb{E}[\hat{\mathbf{g}}_t \mid \mathcal{F}_{t-1}] \rangle +
\\ 7 \eta \Big( \|\hat{\mathbf{g}}_t\|_2^2 - \mathbb{E}[\|\hat{\mathbf{g}}_t\|_2^2 \mid \mathcal{F}_{t-1}] \Big) \Big] \  \mathbf{1}_{\{t \le \Gamma\}}.
\end{multline*}

Its associated martingale difference process
\begin{align} \label{eq:associated_martingale_difference_inequality}
\tilde{\xi}_t-\mathbb{E}[\tilde{\xi}_t \mid \mathcal{F}_{t-1}]
\end{align}
is mean-zero and strictly preserves the martingale difference property.
Moreover, since both the parameter state \(\mathbf{w}_t\) (through \(\Gamma\), defined as the first time the trajectory exits the local region) and the fictitious inputs \(\hat{\mathbf{x}}_t\) are deterministically bounded, the support of \(\hat{\mathbf{g}}_t\) is compact and uniformly bounded. 
Consequently, \(\tilde{\xi}_t\) falls within the scope of standard martingale concentration inequalities, such as Azuma--Hoeffding.

We then condition on the high-probability event \(\mathcal{E}_x \cap \mathcal{E}_{\mathrm{proj}}\). Under this condition, we must carefully control the truncation bias introduced by our coupling argument. Notice that the truncated gradient \(\hat{\mathbf{g}}_t\) is evaluated strictly based on the current step's safety condition. Let us define the local safe event for step \(t\) as \(A_t \triangleq \{\|\mathbf{x}_t\|_2^2 \le cd\} \cap \{|\mathbf{u}_t^\top \mathbf{x}_t| \le c' \log d\}\). The truncation bias can thus be exactly expressed via the local complement event \(A_t^c\):
\[
\mathbf{b}_t
=
\mathbb{E}[\mathbf{g}_t \mid \mathcal{F}_{t-1}]
-
\mathbb{E}[\hat{\mathbf{g}}_t \mid \mathcal{F}_{t-1}]
=
\mathbb{E}[\mathbf{g}_t \mathbf{1}_{A_t^c} \mid \mathcal{F}_{t-1}].
\]

To bound the \(\ell_2\)-norm of this bias, we apply the Cauchy-Schwarz inequality over the conditional expectation:
\[
\|\mathbf{b}_t\|_2 \le \mathbb{E}\left[\|\mathbf{g}_t\|_2 \mathbf{1}_{A_t^c} \mid \mathcal{F}_{t-1}\right] \le \sqrt{\mathbb{E}[\|\mathbf{g}_t\|_2^2 \mid \mathcal{F}_{t-1}]} \cdot \sqrt{\mathbb{P}(A_t^c \mid \mathcal{F}_{t-1})}.
\]

We evaluate the two terms individually. For the first term, by Lemma~\ref{lemma:conditional_variance}, the conditional second moment of the exact gradient is bounded by:
\[
\sqrt{\mathbb{E}[\|\mathbf{g}_t\|_2^2 \mid \mathcal{F}_{t-1}]} = \mathcal{O}(\sqrt{d}).
\]
For the second term, observe that conditioned on \(\mathcal{F}_{t-1}\), the parameter \(\mathbf{u}_t\) is deterministic, while the fresh sample \(\mathbf{x}_t \sim \mathcal{N}(0, \Sigma)\) is strictly independent of the history. Since the spiked covariance \(\Sigma\) has a bounded spectral norm (\(\|\Sigma\|_{\mathrm{op}} = 1+\lambda\)), applying the Hanson-Wright inequality for Gaussian quadratic forms directly yields \(\mathbb{P}(\|\mathbf{x}_t\|_2^2 > cd \mid \mathcal{F}_{t-1}) \le \exp(-\Omega(d))\). Concurrently, since the inner product \(\mathbf{u}_t^\top \mathbf{x}_t\) is a scalar Gaussian variable with bounded variance, the standard Gaussian tail bound implies \(\mathbb{P}(|\mathbf{u}_t^\top \mathbf{x}_t| > c' \log d \mid \mathcal{F}_{t-1}) \le \exp(-\Omega(\log^2 d))\). 
By the union bound, the conditional probability of the complement event \(A_t^c\) is bounded by:
\[
\mathbb{P}(A_t^c \mid \mathcal{F}_{t-1}) \le \exp(-\Omega(d)) + \exp(-\Omega(\log^2 d)) \le \exp(-\Omega(\log^2 d)).
\]

Combining these bounds, the single-step truncation bias is exponentially small:
\[
\|\mathbf{b}_t\|_2 \le \mathcal{O}(\sqrt{d}) \exp(-\Omega(\log^2 d)).
\]
Because this exponential tail decay strictly dominates the polynomial dimension growth \(\mathcal{O}(\sqrt{d})\) as \(d \to \infty\), the cumulative truncation bias is safely absorbed into the final \(\tilde{\mathcal{O}}(d^{-1/2})\) residual term. 
This rigorously justifies that we can bound the cumulative noise of the true trajectory by analyzing its strictly bounded fictitious counterpart \(\tilde{\xi}_t\). 
Consequently, we can now safely apply the Azuma-Hoeffding inequality to \(\tilde{\xi}_t\) to define the high-probability uniform noise bound event \(\mathcal{E}_{\mathrm{noise}}\):
\begin{equation} \label{eq:noise_event}
    \mathcal{E}_{\text{noise}} \triangleq \left\{ \max_{1 \le t \le T} \left| \sum_{j=0}^{t-1} \eta \tilde{\xi}_j \right| \le \tilde{\mathcal{O}}(d^{-1/2}) \right\},
\end{equation}
where $\eta=\delta/d=\alpha \delta/T$.
By applying the concentration inequality for martingales Eq.~\eqref{eq:associated_martingale_difference_inequality} combined with a union bound over $T$ steps, this event occurs with overwhelming probability: $\mathbb{P}(\mathcal{E}_{\text{noise}}) \ge 1 - \exp(-\Omega(\log^2 d))$.

We utilize mathematical induction pathwise conditioned on the joint event $\mathcal{E}_{\text{global}} \triangleq \mathcal{E}_{\text{noise}} \cap \mathcal{E}_x \cap \mathcal{E}_{\text{proj}} \cap \mathcal{E}_{\text{init}}$. 
\textbf{Base Case:} As established by Eq.~\eqref{bound:init_case}, the initial distance satisfies $\Delta_0 < R_{\text{max}} - 2\epsilon_0 < R_{\text{max}}$. Thus, $\Gamma > 0$ holds trivially.
\textbf{Inductive Step:} Assume for an arbitrary $t < T$, the trajectory has not escaped, i.e., $t < \Gamma$ strictly holds, which implies $\Delta_k \le R_{\text{max}}$ for all $k \le t$. 
Under this inductive hypothesis, for all $j \le t$, the indicator $\mathbf{1}_{j \le \Gamma} = 1$ identically, establishing the exact pathwise equivalence $\xi_j = \tilde{\xi}_j$. We can therefore algebraically unroll the original, unstopped recurrence Eq.~\eqref{eq:stochastic_update} up to step $t+1$:
\begin{equation} \label{eq:pathwise_unroll}
    \Delta_{t+1} \le (1-\eta\mu)^{t+1} \Delta_0 + \left(1 - (1-\eta\mu)^{t+1}\right) D_\infty + \sum_{j=0}^{t} (1-\eta\mu)^{t-j} \eta \tilde{\xi}_j.
\end{equation}

By applying Abel's Lemma (summation by parts) exploiting the monotonic discounted weights $(1-\eta\mu)^{t-j} \le 1$, the weighted noise summation in Eq.~\eqref{eq:pathwise_unroll} is absolutely bounded by $2 \max_{k \le t} |\sum_{i=0}^k \eta \tilde{\xi}_i|$. Conditional on $\mathcal{E}_{\text{noise}}$ from Eq.~\eqref{eq:noise_event}, this is rigorously bounded by $\tilde{\mathcal{O}}(d^{-1/2})$.
Substituting the strictly bounded stationary limit $D_\infty \le R_{\text{max}} - \epsilon_d$ into the unrolled trajectory yields:
\begin{align} \label{eq:final_distance_bound}
    \Delta_{t+1} &\le (1-\eta\mu)^{t+1} (R_{\text{max}}-\epsilon_0) + \left(1 - (1-\eta\mu)^{t+1}\right) (R_{\text{max}} - \epsilon_d) + \tilde{\mathcal{O}}(d^{-1/2}) \notag \\
    &\le R_{\text{max}} - \min(\epsilon_d,\epsilon_0) + \tilde{\mathcal{O}}(d^{-1/2}).
\end{align}

Because $\epsilon_d = \Theta(d^{-c})$ with $c \in (0, \frac{1}{2})$, while the noise fluctuation scales as $\tilde{\mathcal{O}}(d^{-1/2})$.
As $d \to +\infty$, the noise fluctuation is bounded by the buffer: $\tilde{\mathcal{O}}(d^{-1/2}) < \min(\epsilon_d,\epsilon_0)$.
Consequently, $\Delta_{t+1} < R_{\text{max}}$. 
By induction, conditioned on the event $\mathcal{E}_{\text{global}}$, the parameter state does not violate the boundary condition up to time $T$. Thus, we establish that $\Gamma > T$ holds pathwise on the high-probability event $\mathcal{E}_{\text{global}}$.
Because the state space $\mathbb{S}^{d-1}$ is bounded, $\Delta_T \le 4$. 
We compute the unconditional expected error strictly using the Law of Total Expectation over the joint uncorrupted event $\mathcal{E}_{\text{global}} \triangleq \mathcal{E}_{\text{noise}} \cap \mathcal{E}_x \cap \mathcal{E}_{\text{proj}} \cap \mathcal{E}_{\text{init}}$:
\begin{align} \label{eq:total_expectation}
\mathbb{E}[\Delta_T] &= \mathbb{E}[\Delta_T \mathbf{1}_{\mathcal{E}_{\text{global}}}] + \mathbb{E}[\Delta_T \mathbf{1}_{\mathcal{E}_{\text{global}}^c}] \notag \\
& \le 1 \cdot \left[ \left(1-\frac{\delta\mu}{d}\right)^T (2-2\tau) + \left(1 - \left(1-\frac{\delta\mu}{d}\right)^T\right) D_\infty + \tilde{\mathcal{O}}(d^{-1/2}) \right] + 4 \cdot \mathbb{P}(\mathcal{E}_{\text{global}}^c).
\end{align}

The failure probability is bounded by the union bound: $\mathbb{P}(\mathcal{E}_{\text{global}}^c) \le \mathbb{P}(\mathcal{E}_{\text{noise}}^c) + \mathbb{P}(\mathcal{E}_x^c) + \mathbb{P}(\mathcal{E}_{\text{proj}}^c) + \mathbb{P}(\mathcal{E}_{\text{init}}^c)$. 
Recall that these individual failure events represent extreme, low-probability anomalies:
\begin{itemize}
    \item $\mathcal{E}_{\text{init}}^c$ in Eq.~\eqref{bound:init_case}: Failure of the spectral initialization to reach the effective correlation threshold, with $\mathbb{P}(\mathcal{E}_{\text{init}}^c) \le \exp(-\Omega(d))$.
    \item $\mathcal{E}_x^c$ and $\mathcal{E}_{\text{proj}}^c$ in Eq.~\eqref{bound_prob_x_proj}: Encounters with anomalously large inputs or one-dimensional projections, with $\mathbb{P}(\mathcal{E}_x^c \cup \mathcal{E}_{\text{proj}}^c) \le \exp(-\Omega(\log^2 d))$.
    \item $\mathcal{E}_{\text{noise}}^c$ in~Eq.~\eqref{eq:noise_event}: The stochastic noise breaking the established martingale concentration bounds, with $\mathbb{P}(\mathcal{E}_{\text{noise}}^c) \le \exp(-\Omega(\log^2 d))$.
\end{itemize}

Consequently, the total failure probability $\mathbb{P}(\mathcal{E}_{\text{global}}^c)$ decays at least as fast as the super-polynomial rate $\exp(-\Omega(\log^2 d))$. 
Returning to the total expectation evaluated in Eq.~\eqref{eq:total_expectation}, the final bound is a weighted sum of two scenarios: the successful trajectory (which inherently contains the polynomially decaying martingale noise residual $\tilde{\mathcal{O}}(d^{-1/2})$ established in Eq.~\eqref{eq:noise_event}), and the worst-case failure penalty bounded by $4 \cdot \mathbb{P}(\mathcal{E}_{\text{global}}^c)$. 
As $d \to +\infty$, the super-polynomial decay of the failure penalty (which scales as $d^{-\Omega(\log d)}$) vanishes significantly faster than the polynomial decay $\tilde{\mathcal{O}}(d^{-1/2})$. 
Because this slower-decaying polynomial term strictly dominates the asymptotic sum, the infinitesimally small failure penalties are mathematically absorbed into the $\tilde{\mathcal{O}}(d^{-1/2})$ residual. This yields the final decomposed expected generalization bound:
\begin{equation} \label{eq:final_theorem_bound}
\mathbb{E}[\|\mathbf{w}_T - \mathbf{v}_*\|_2^2] \le \left(1 - \frac{\delta \mu}{d}\right)^T (2 - 2\tau) + \left[1 - \left(1 - \frac{\delta \mu}{d}\right)^T\right] \left( \frac{7 \delta G}{\mu} + \frac{2\phi^2}{\mu^2} \right) + \tilde{\mathcal{O}}(d^{-1/2}).
\end{equation}

\end{proof}

\section{Additional Experimental Details} \label{appendix:additional}

\subsection{Synthetic Experiments Details}
\label{app:synthetic_details}

We provide the detailed configurations for the synthetic experiments presented in Section~\ref{sec:synthetic_experiments}. 

\textbf{Base Environment and Activation Function.} We conduct the experiments within a precisely controlled high-dimensional setting with the dimension $d=200$. The population risk and expected gradients are evaluated using $N_{\text{pop}} = 50,000$ samples. For our synthetic experiments, we explicitly employ the shifted Hermite polynomial $f(z) = z^3 - 3z$ to capture the learning of hard concepts ($k \ge 3$). 
While polynomials do not possess globally bounded derivatives, they serve as the canonical basis for characterizing tasks with information exponent $k=3$. In practice, due to the concentration of the Gaussian measure and the spherical constraint $\|\mathbf{w}\|_2=1$, the activation inputs are effectively confined within a compact interval where the derivatives are locally bounded. This ensures that the optimization dynamics remain stable and qualitatively aligned with our theoretical predictions, effectively behaving as a smooth, truncated version of the polynomial.

\textbf{Data Generation and Model Entities.} The inputs are drawn from a spiked Gaussian distribution $\mathcal{N}(\mathbf{0}, \mathbf{I}_d + \lambda \mathbf{v}\mathbf{v}^\top)$ with the latent signal strength $\lambda = 15$. To mathematically control the alignment between pre-training and downstream tasks, the latent direction $\mathbf{v}$ is constructed as $\mathbf{v} = \rho \mathbf{v}_* + \sqrt{1 - \rho^2} \mathbf{u}$, where $\mathbf{u}$ is a random unit vector orthogonal to the ground-truth parameter $\mathbf{v}_*$. Unsupervised pre-training is simulated via PCA on $N_{\text{pre}} = 10,000$ samples. Throughout all synthetic experiments, the weak supervisor is fixed with a constant, suboptimal geometric correlation to the ground truth, denoted as $\gamma = \mathbf{v}_{\text{weak}}^\top \mathbf{v}_* = 0.5$.

For \textbf{Experiment 1}, which investigates the spatial emergence of W2SG, we systematically vary the intrinsic task alignment $\rho$ across uniformly distributed values in the interval $[0, 1]$ (i.e., $\rho \in \{0, 0.05, 0.1, \dots, 1\}$). We evaluate the generalization error of the strong student after $30$ steps of W2SG fine-tuning with a learning rate of $\eta = 4 \times 10^{-5}$.

For \textbf{Experiment 2}, which investigates the temporal optimization dynamics across different pre-training qualities, we simulate multiple fine-tuning trajectories. We explicitly construct four pre-trained initialization vectors representing varying degrees of latent concept extraction, corresponding to initial correlations of $\tau \in \{0.1, 0.4, 0.7, 0.9\}$ with the ground truth $\mathbf{v}_*$. From each distinct starting point, the strong student undergoes $2000$ steps of spherical PGD fine-tuning on the weak labels with a learning rate of $\eta = 1 \times 10^{-5}$. We explicitly trace the Euclidean distance $\|\mathbf{w}_t - \mathbf{v}_*\|_2$ at each step to observe how the quality of pre-training universally governs the drift-to-bias phase transition and the ultimate stationary error limit.

\subsection{Illustrations} \label{appendix:illustration}

To provide a more intuitive geometric understanding of the theoretical mechanisms discussed in the main text, we present two conceptual illustrations in this section. 

First, Figure~\ref{fig:effective_region} visualizes the fundamental difference between random initialization and a pre-training-induced ``warm start'' in high-dimensional space. It intuitively explains why a randomly initialized model remains trapped in an orthogonal state, and why pre-training is strictly necessary to place the parameters within the effective region. 

Second, once the model enters this effective region, Figure~\ref{fig:drift_bias} illustrates the subsequent W2SG fine-tuning trajectory. It explicitly depicts the two distinct optimization phases characterized by our Theorem~\ref{thm:wtsg}: the initial deterministic contraction towards the ground truth (the Drift Regime) and the eventual saturated oscillation bounded by the weak supervisor's inherent noise (the Bias Regime).

\begin{figure}[t]
\centering
\includegraphics[width=0.7\textwidth]{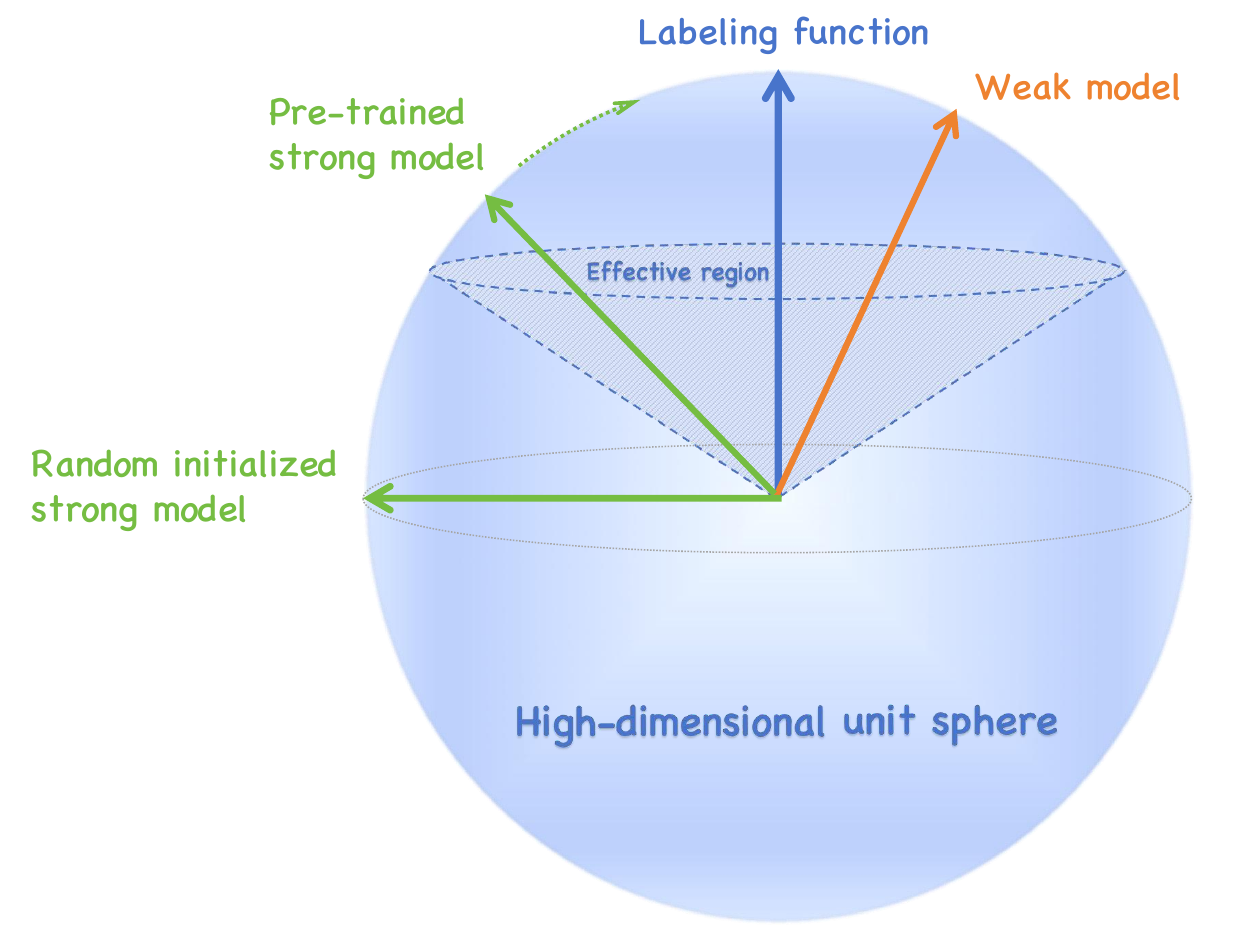}
\caption{In W2SG, the ideal scenario is for the strong model to approximate the ground-truth labeling function as closely as possible. In high-dimensional space, a randomly initialized strong model often remains orthogonal to this labeling function, which significantly hinders its convergence toward the global minimum. Conversely, pre-training provides a ``warm start'' that better prepares the model to extract signal from weak supervision, thereby increasing the likelihood of reaching the global optimum.}
\label{fig:effective_region}
\vspace{-10pt}
\end{figure}

\begin{figure*}[t]
\centering
\vspace{-5pt}
\includegraphics[width=0.95\textwidth]{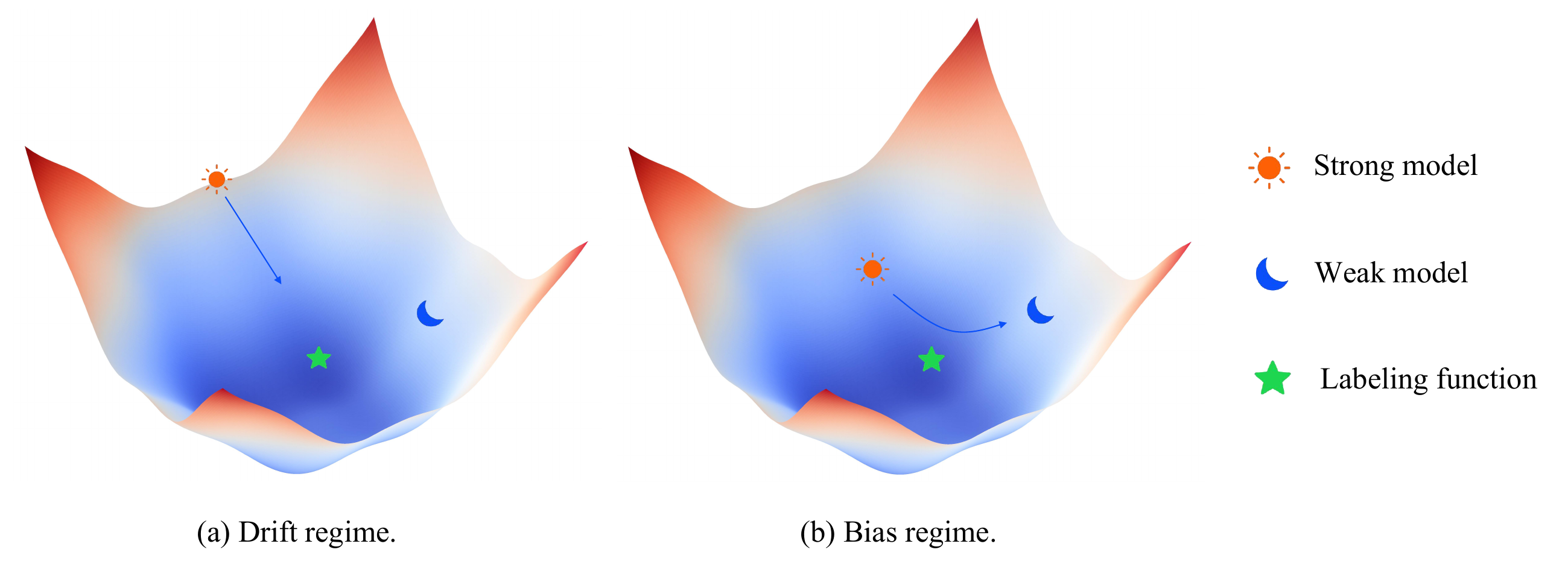}
\vspace{-5pt}
\caption{The optimization dynamics in the effective region. 
(a) Drift regime: When pre-training initializes the strong student within the basin of attraction, the gradient signal towards the ground truth dominates the noise. The model undergoes deterministic contraction towards the optimum. 
(b) Bias regime: As the student approaches the ground truth, the gradient signal diminishes. The optimization becomes dominated by the systematic bias introduced by the weak supervisor, causing the student to converge to a stationary distribution with a non-vanishing error floor rather than the exact ground truth.}
\label{fig:drift_bias}
\vspace{-10pt}
\end{figure*}

\subsection{Validation of Assumption~\ref{assum:pretraining_w2sg}} \label{appendix:syn_exp_validate_assumption}

\begin{figure*}[t]
\centering
\includegraphics[width=0.66\textwidth]{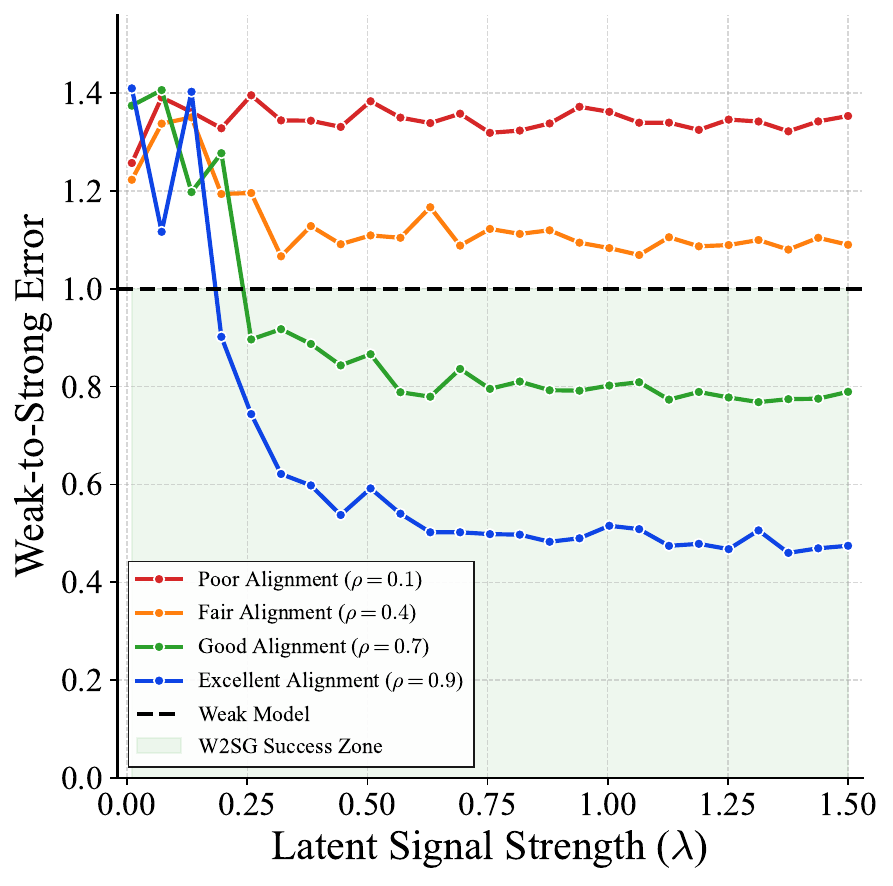}
\caption{Validation of Assumption~\ref{assum:pretraining_w2sg}(1).}
\label{fig:cond1_lambda}
\end{figure*}

\noindent \textbf{Condition 1: Detectable Latent Structure $\lambda$.} 
To verify the necessity of a detectable latent signal, we simulate the W2SG error across varying latent signal strengths ($\lambda$) under different levels of intrinsic task alignment ($\rho$). The weak supervisor's quality is kept constant. As illustrated in Figure~\ref{fig:cond1_lambda}, the system exhibits a clear phase transition. When $\lambda$ is below the theoretical threshold, PCA fails to extract meaningful features from the pre-training data, resulting in a high, flat error rate regardless of $\lambda$. Once $\lambda$ surpasses the threshold, the error drops precipitously into the W2SG success zone. Furthermore, the multiple trajectories reveal a critical coupling effect: a large $\lambda$ is a necessary but insufficient condition. If the intrinsic alignment is excessively poor (e.g., the red curve where $\rho=0.1$), the strong model fails to generalize even with perfectly detectable signals, validating Assumption~\ref{assum:pretraining_w2sg}(1).

\noindent \textbf{Condition 2: Intrinsic Task Alignment $\rho$.} 
We investigate the impact of pre-training quality by varying the intrinsic task alignment ($\rho$) under different signal-to-noise ratios ($\lambda$). Figure~\ref{fig:cond2_rho} demonstrates that $\rho$ serves as the decisive ``blessing'' for W2SG. For datasets with adequate signal strength (e.g., $\lambda \ge 1.0$), the W2SG error monotonically and smoothly decreases as $\rho$ increases, eventually breaking through the weak supervisor's baseline and achieving remarkable generalization. Conversely, under a poor signal regime ($\lambda=0.05$), the pre-training yields pure noise, causing the error to persistently remain at a catastrophic level ($\approx 1.41$) regardless of how $\rho$ varies. This decoupling firmly empirically validates Assumption~\ref{assum:pretraining_w2sg}(2): adequate intrinsic task alignment is the fundamental prerequisite for initiating the W2SG drift regime.

\begin{figure*}[t]
\centering
\includegraphics[width=0.66\textwidth]{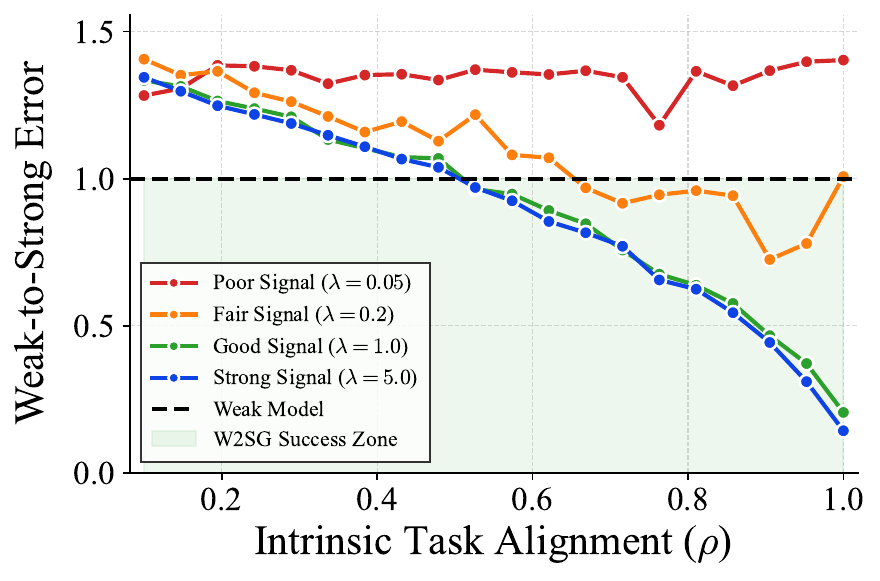}
\caption{Validation of Assumption~\ref{assum:pretraining_w2sg}(2).}
\label{fig:cond2_rho}
\vspace{-10pt}
\end{figure*}

\begin{figure*}[t]
\centering
\includegraphics[width=0.66\textwidth]{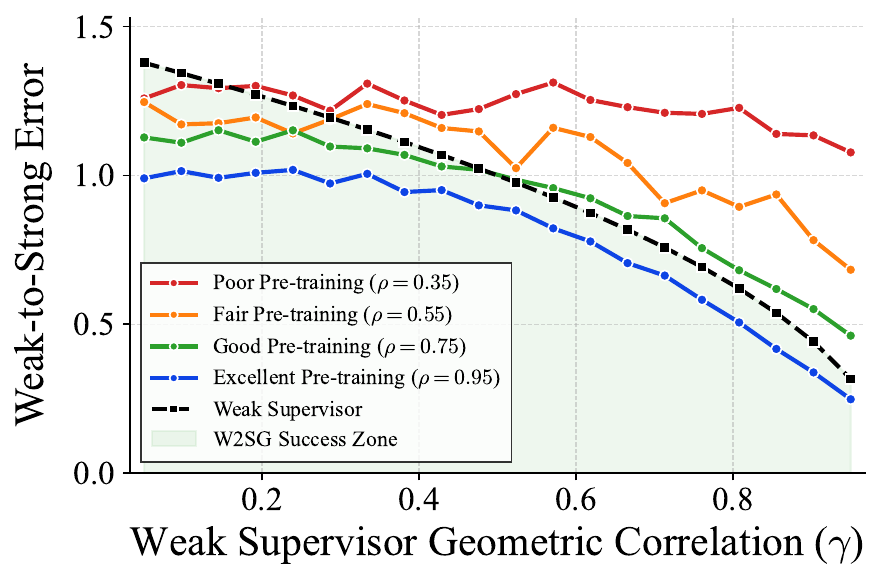}
\caption{Validation of Assumption~\ref{assum:pretraining_w2sg}(3).}
\label{fig:cond3_gamma}
\end{figure*}

\noindent \textbf{Condition 3: Informative Weak Supervisor $\gamma$ ($\phi$).} 
Although Assumption~\ref{assum:pretraining_w2sg}(3) formally imposes a theoretical constraint on the systematic bias $\phi$, directly isolating $\phi$ as an independent control variable is intractable in simulations, as it emerges a posteriori from the complex interplay of non-linear activations and data distributions. 
To circumvent this, we utilize the geometric correlation $\gamma$ between the weak supervisor and the ground truth as a rigorous generative proxy. Given that a larger $\gamma$ deterministically induces a smaller systematic bias $\phi$, sweeping $\gamma$ allows for an equivalent empirical evaluation of this condition. 
We sweep $\gamma$ from $0.05$ to $0.95$ across models initialized with varying degrees of pre-training alignment ($\rho$). 
As depicted in Figure~\ref{fig:cond3_gamma}, the overall W2SG error exhibits a consistent downward trajectory as the weak supervisor becomes more informative (i.e., larger $\gamma$, implying smaller $\phi$). 
This confirms that the weak model effectively determines the optimization bottleneck during the bias regime. Crucially, the fine-tuning dynamics drastically diverge based on pre-training quality. Models with robust pre-training (green and blue curves) exhibit a pronounced ``diving'' effect, rapidly piercing through the weak supervisor's baseline and settling deep within the W2SG success zone. 
Conversely, poorly pre-trained models (red curve) suffer from a sluggish descent, ultimately failing to extract meaningful guidance even from a highly competent weak teacher. 
This strictly corroborates Assumption~\ref{assum:pretraining_w2sg}(3): while an informative weak supervisor dictates the theoretical upper bound of the strong model's performance, it is fundamentally incapable of rescuing a model that lacks proper structural alignment from pre-training.

\noindent \textbf{Condition 4: Stable Learning Rate ($\eta$).} 
Finally, we examine the optimization dynamics by varying the learning rate ($\eta$) on a logarithmic scale. We deliberately select a range of initializations, including a ``Poor Init'' ($\rho=0.4$) whose initial error strictly exceeds the weak model's baseline, which is shown in Figure~\ref{fig:cond4_lr}. 
In the extremely small $\eta$ regime, models suffer from underfitting and strictly remain at their respective initialization errors. Within the optimal ``sweet spot'' ($10^{-4} \le \eta \le 10^{-3}$), all models seamlessly dive into the W2SG success zone. Strikingly, even the ``Poor Init'' model successfully plunges below the weak model baseline, dispelling the misconception that W2SG is purely an artifact of good initialization. 
However, this successful generalization is highly fragile with respect to optimization stability. As $\eta$ crosses the threshold ($\approx 10^{-2}$), all trajectories—even those perfectly poised to succeed with excellent pre-training—universally and abruptly spike to the maximum orthogonal error ($\approx 1.4$). 
This stark contrast emphasizes that an excessively large learning rate will deterministically obliterate W2SG, overriding any benefits from good initialization or informative weak supervision. Conceptually, this phenomenon is analogous to the overfitting dilemma commonly observed in standard W2SG settings, where an excessively large number of training epochs causes the strong model to memorize the weak supervisor's noise and lose generalization~\citep{burns2023weak}. 
Here, violating the stability threshold $\eta$ prematurely shatters the optimization dynamics that aligns with the theoretical constraints posited in Assumption~\ref{assum:pretraining_w2sg}(4).

\begin{figure*}[t]
\centering
\includegraphics[width=0.66\textwidth]{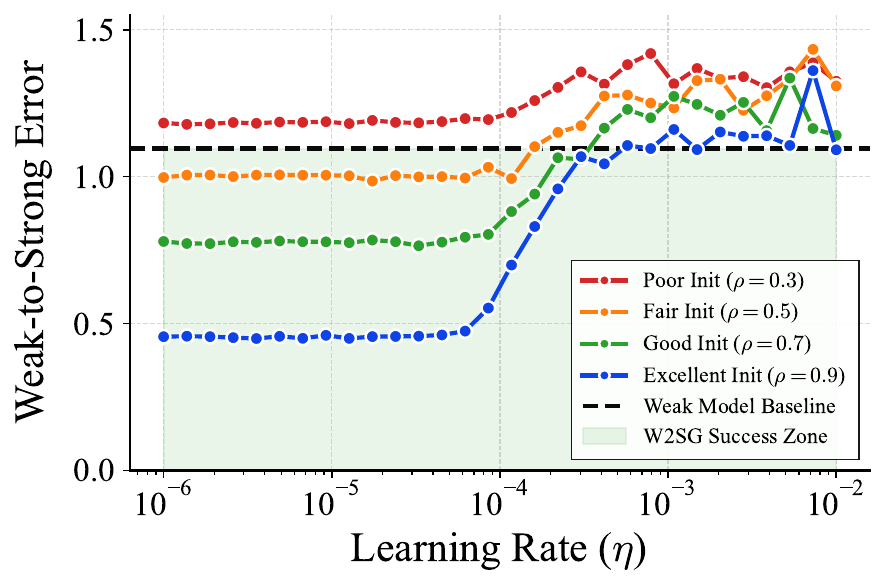}
\caption{Validation of Assumption~\ref{assum:pretraining_w2sg}(4).}
\label{fig:cond4_lr}
\vspace{-10pt}
\end{figure*}

\subsection{Additional Synthetic Experiments: Probing the Empirical Effective Region} \label{app:synthetic_additional}

To further bridge the gap between our abstract theoretical bounds and empirical observations, we design a fine-grained numerical simulation to probe the loss landscape during the pre-training phase. Specifically, we aim to validate the remark in Appendix~\ref{appendix:bound_radius}: 
while our analytical derivation guarantees a strictly dimension-independent but conservative worst-case radius (typically spanning only a few degrees), the empirical basin of attraction is significantly broader.

\noindent \textbf{Experimental Setup and Metric Definitions.}
We simulate the unsupervised pre-training dynamics (via power iteration on the spiked covariance matrix) using a bounded "hard concept" activation function, $f(x) = 2\tanh^3(x)$. To isolate the impact of the latent signal strength, we fix the intrinsic task alignment at a challenging moderate level ($\rho = 0.65$) and the weak supervisor quality at $\gamma = 0.65$. We then evaluate the landscape across three distinct latent signal strengths: $\lambda \in \{5.0, 15.0, 30.0\}$. 
At each pre-training step, we explicitly compute the local strong convexity ($\mu$) and the systematic noise bias ($\phi$). Furthermore, to geometrically quantify the optimization basin, we define two macroscopic metrics:
(1) \textbf{Effective Region Angle ($\theta$):} Computed as $\theta = \arccos(\tau) \times \frac{180^\circ}{\pi}$, where $\tau = \langle \mathbf{w}, \mathbf{v}_* \rangle$. It represents the geodesic distance on the hypersphere from the current model parameter to the ground truth.
(2) \textbf{Tangential Signal Pull:} To measure the explicit restoring force driving the model toward the target, we project the negative true gradient onto the unit tangent vector pointing strictly toward $\mathbf{v}_*$: $\langle -\nabla\mathcal{L}_{\text{true}}, \mathbf{t} \rangle$, where $\mathbf{t} = \frac{\mathbf{v}_* - \tau \mathbf{w}_{\text{pre}}}{\|\mathbf{v}_* - \tau \mathbf{w}_{\text{pre}}\|_2}$.

\noindent \textbf{Validating the Broad Empirical Radius and Landscape Stabilization.}
As illustrated in Figure~\ref{fig:landscape_lambda}, the randomly initialized parameters are deterministically pulled into the latent subspace by Step 4, causing all landscape metrics to lock into steady values. Notably, the model settles into a stable geometric angle ranging from $49.4^\circ$ to $55.1^\circ$. This visually confirms that the empirical effective region is remarkably broad ($\sim 50^\circ$), vastly exceeding the tight, worst-case theoretical bounds, indicating that the actual basin of attraction is much more accommodating.

A key physical mechanism observed across the three settings is that as the latent signal strength $\lambda$ increases, the absolute magnitudes of the Tangential Pull, Noise Bias ($\phi$), and Strong Convexity ($\mu$) all concurrently decrease. This happens because a stronger $\lambda$ amplifies the variance of the input features, pushing a larger fraction of the pre-activation values into the flat, bounded tails of the non-linear activation function. Consequently, the gradients are naturally dampened. This dampening mechanism structurally compresses the overall scale of the landscape, squeezing both the restorative signal pull and the adversarial noise bias.

\begin{figure}[t]
    \centering
    \includegraphics[width=\textwidth]{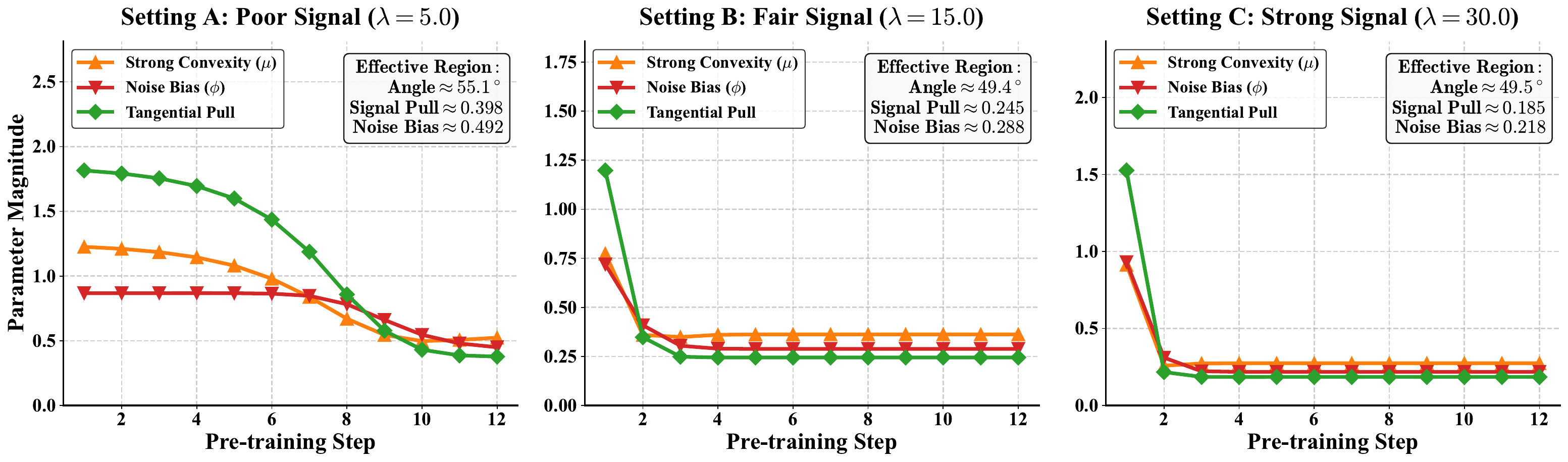}
    \caption{The evolution of landscape metrics across different latent signal strengths ($\lambda$). While the theoretical worst-case radius is highly restrictive, the fixed pre-training alignment drops the model into a broad empirical basin ($\sim 50^\circ$). As $\lambda$ increases, the gradients are naturally dampened by the bounded activation, which effectively compresses the Noise Bias ($\phi$), allowing the Strong Convexity ($\mu$) to robustly preserve the basin's integrity.}
    \label{fig:landscape_lambda}
    \vspace{-10pt}
\end{figure}

\noindent \textbf{The Interplay of Signal and Noise.}
The cross-figure comparison explicitly demonstrates how a sufficient latent signal preserves the basin's integrity despite the large angular distance:
(1) \textbf{Setting A (Poor Signal, $\lambda=5.0$):} With a weak latent signal, the input variance is insufficient to reach the dampened regions of the activation function, leaving the landscape highly volatile at a wider boundary of $\theta \approx 55.1^\circ$. Here, the Noise Bias ($\phi \approx 0.492$) is dangerously high, outpacing the Tangential Pull ($\approx 0.398$) and encroaching severely upon the Strong Convexity ($\mu$). In this regime, the weak supervisor's errors easily dominate the optimization trajectory.
(2) \textbf{Settings B \& C (Fair to Strong Signal, $\lambda=15.0, 30.0$):} As $\lambda$ increases, the enhanced signal tightens the stable angle to $\sim 49.5^\circ$. More importantly, the gradient dampening acts as a natural noise filter. While the absolute Tangential Pull drops (to $0.245$ and $0.185$, respectively), the Noise Bias $\phi$ is heavily compressed simultaneously (dropping to $0.288$ and $0.218$). 
Crucially, in the strong signal regimes (Settings B and C), the Strong Convexity $\mu$ (the orange curves) successfully maintains a strict and robust margin above the Noise Bias $\phi$ (the red curves), satisfying the fundamental condition $\mu > \phi$ required for an effective basin. This demonstrates that even at a geometrically broad radius of $\sim 50^\circ$, a sufficiently strong latent signal ($\lambda$) naturally suppresses systematic bias through the bounded activation, robustly preserving the structural integrity of the basin required for W2SG.

\subsection{LLM Experimental Setup} \label{appendix:exp_setup}

\noindent \textbf{Dataset.}
Consistent with prior research~\citep{burns2023weak,yang2024super}, our experiments focus on reward modeling tasks across two primary alignment dimensions: harmlessness and helpfulness.
For harmlessness, we adopt the CAI-Harmless benchmark~\citep{bai2022constitutional} following the setup in~\citet{yang2024super}.
For helpfulness, we utilize a subset of single-turn dialogues from the HH-RLHF dataset~\citep{bai2022training}.
We partition each dataset into three distinct subsets of 4,000 samples each:
(1) \textbf{Ground-truth set}: Used to fine-tune base models and establish the strong ceiling performance.
(2) \textbf{Weak supervision set}: A held-out set where the weak model generates labels to supervise the strong model.
(3) \textbf{Test set}: Reserved for evaluating all models.
We structure every example as a triplet containing the user's input, the chosen response, and the rejected response.

\noindent \textbf{Model.}
Our analysis spans 317 checkpoints of OLMo-7B~\citep{Groeneveld2023OLMo} and 154 pre-training checkpoints of Pythia-6.9B~\citep{biderman2023pythia}. 
While we use the full set of released Pythia-6.9B checkpoints, for OLMo-7B, we select the first 100 checkpoints and then adopt a sparse sampling rate (every 5 checkpoints) up to the 962nd, ensuring a comprehensive yet computationally manageable study.
These intermediate pre-training checkpoints allow us to trace the stage of model parameters along the whole pre-training process.
We attach a linear projection head for each model to enable logit predictions. 
The model is trained to generate a soft value between 0 to 1 for each sample.

\noindent \textbf{Weak-to-strong training and evaluation.}
We train the strong ceiling models using standard cross-entropy loss on the ground-truth dataset, while the weak-to-strong models are fine-tuned on the weak supervision set. All models are trained for a single epoch to mitigate overfitting, following \citet{burns2023weak}. The optimization configuration includes a batch size of 16, a learning rate of $10^{-5}$, and a \texttt{max\_seq\_len} of 512. Performance is reported as accuracy on the held-out test set.

\subsection{Extended Results on W2SG Pre-training Dynamics} \label{appendix:extended_w2sg}

To complement the findings in the main paper, we provide extended empirical results tracking the evolution of W2SG dynamics during pre-training on the CAI-Harmless dataset. We systematically evaluate the fine-tuning performance of intermediate checkpoints from both the Pythia-6.9B and OLMo-7B model families. As shown in Figure~\ref{fig:w2s_cai}, the optimization trajectories consistently exhibit the macroscopic phase transitions predicted by our theory: an initial failure phase (noise trap), followed by a rapid performance ascent (drift regime), and ultimately saturating into a stationary fluctuation phase bottlenecked by the weak supervisor's inherent bias.

\begin{figure*}[t]
\centering
\begin{subfigure}{\textwidth}
    \centering
    \includegraphics[width=0.75\textwidth]{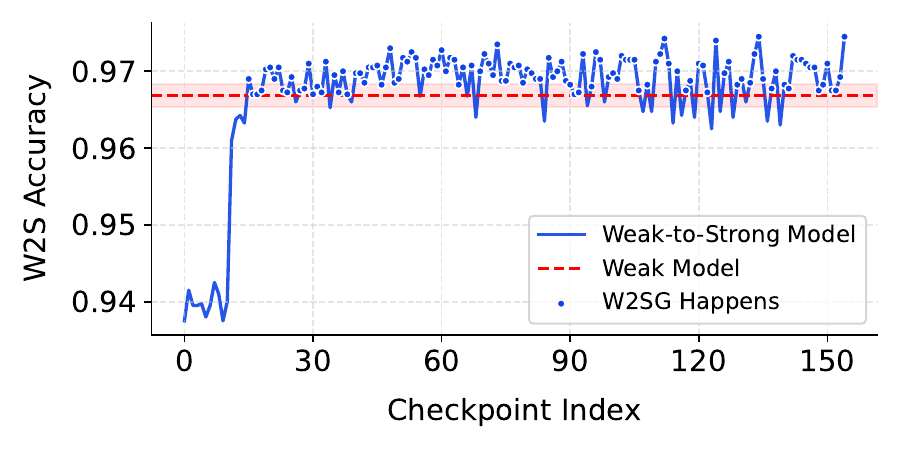}
    \caption{Pythia-6.9B}
\end{subfigure}
\vspace{20pt}
\begin{subfigure}{\textwidth}
    \centering
    \includegraphics[width=0.8\textwidth]{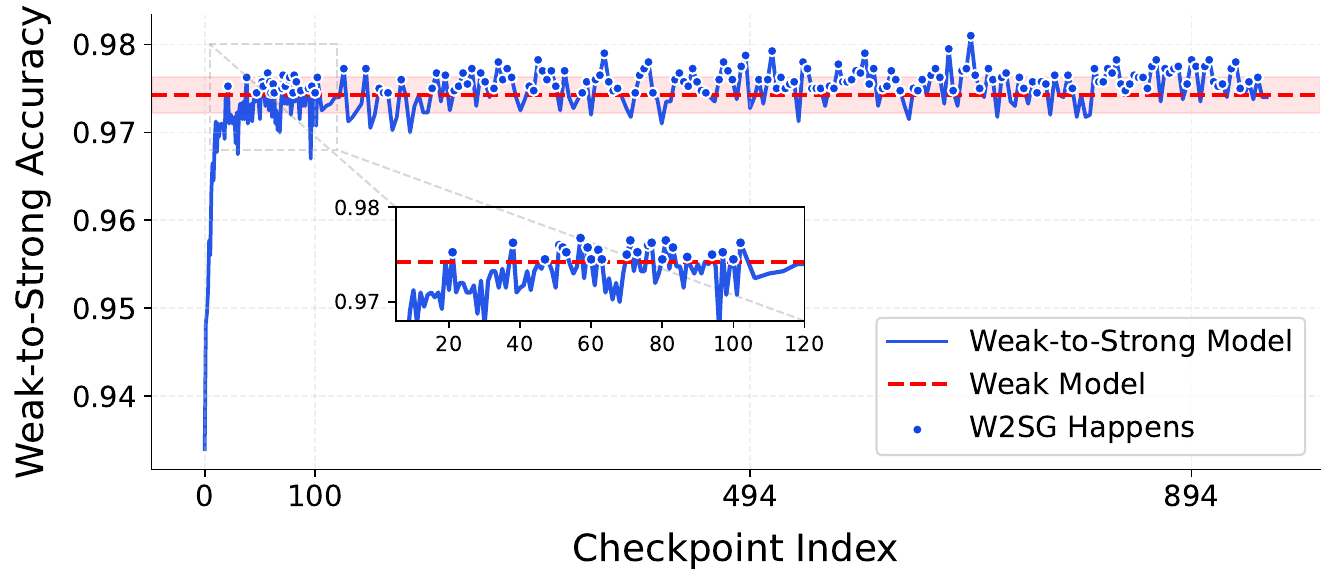}
    \caption{OLMo-7B}
\end{subfigure}
\vspace{-25pt}
\caption{Evolution of W2SG dynamics during pre-training on the CAI-Harmless dataset. 
The plot reports the test accuracy of the strong student model fine-tuned on weak labels. 
The solid curve represents the mean performance aggregated over 5 random seeds, while the shaded region indicates the standard deviation. 
The dashed horizontal line denotes the weak supervisor's performance level. }
\label{fig:w2s_cai}
\vspace{-5pt}
\end{figure*}

\subsection{W2SG and Linear Representation Hypothesis} \label{appendix:lrh}

\subsubsection{Linear Probing Methodology}
To quantify the linear separability of the learned features, we employ the standard linear probing technique~\citep{alain2016understanding,belinkov2022probing}. 
Specifically, given each original dataset consisting of sentences and the corresponding class labels, we feed the sentence into LLMs and collect the corresponding activations of the last token~\cite{li2023inferencetime,gurnee2024language} for each layer. The activation dataset $\mathcal{D}=\{( \mathbf{x}_i ,y_i) \}_{i=1}^N$ is constructed with the activations $\mathbf{x}_i \in \mathbb{R}^d$ and the corresponding binary labels $y_i \in \{0,1\}$.
For each dataset, every layer of each pre-training checkpoint produces an activation dataset. 
Therefore, there are $L \times C$ activation datasets for all $L$ layers across $C$ checkpoints. 
We randomly split each activation dataset into training and test sets by 4:1, and fit a binary linear classifier on the training set. 
We train a classifier for each activation dataset, which yields $L \times C$ classifiers. We compute the accuracy on the test set of these classifiers.
In our analysis, we primarily report the performance of the final layer, as it aggregates the comprehensive contextual information required for prediction~\citep{li2025tracing}.

\subsubsection{Correlation Results and Quantitative Statistics} \label{appendix:exp_correlation}

To investigate the relationship between the linear separability of learned features and W2SG performance, we first examine the probing accuracy trajectories across the pre-training process. As illustrated in Figure~\ref{fig:linear_probe_helpful} (HH-RLHF) and Figure~\ref{fig:linear_probe_cai} (CAI-Harmless), the linear probe accuracy consistently increases during the initial pre-training phase before transitioning into sustained fluctuations. This suggests that models in early pre-training stages are already capable of encoding diverse concepts from the dataset in a linearly decodable manner, aligning with the findings in~\citet{qian-etal-2024-towards}.

Comparing these trajectories with the W2SG evolution reveals a striking macroscopic similarity: both linear probe and weak-to-strong model accuracies exhibit an initial sharp improvement followed by a saturated period. To formally quantify this observation, we visualize the correlation between linear probing accuracy and W2SG performance during the first 50 checkpoints in Figure~\ref{fig:correlation_helpful} (HH-RLHF) and Figure~\ref{fig:correlation_cai} (CAI-Harmless). Across both Pythia and OLMo model families and on both distinct datasets, we observe a robust and consistent positive correlation.

To rigorously substantiate this visual evidence, we provide detailed quantitative statistics in Table~\ref{tab:correlation_stats}, reporting the coefficient of determination ($R^2$), Spearman's correlation coefficient ($\rho$), and the associated p-values. The consistently high $R^2$ values (ranging from $0.44$ to $0.89$) and the extremely low, statistically significant p-values reinforce our correlation hypothesis. Furthermore, to account for potential confounds, we report the correlation between standard evaluation loss and final W2SG accuracy in Table~\ref{tab:correlation_new}. A direct comparison demonstrates that linear probes---which explicitly capture the geometric alignment extracted during pre-training---serve as a substantially more predictive indicator of W2SG performance than standard evaluation loss.

These results suggest that as pre-training progresses, the LLM's internal representations increasingly evolve towards a linearly separable manifold. This structured geometric space serves as the necessary foundation enabling the strong model to generalize effectively from weak supervision. The underlying mechanism can be interpreted through the lens of optimization: W2SG training essentially acts as a high-dimensional, non-linear ``probing'' process that navigates complex parameter geometry, yet remains fundamentally anchored by the linear decodability of the base model's representations. Consequently, if the linear representation hypothesis~\citep{park2024the} holds, metrics derived from mechanistic interpretability could serve as potent tools for demystifying W2SG, potentially enabling the ex-ante estimation of W2SG feasibility before undertaking computationally exhaustive fine-tuning.



\begin{figure*}[t]
\centering
\begin{subfigure}{0.49\textwidth}
    \centering
    \includegraphics[width=\textwidth]{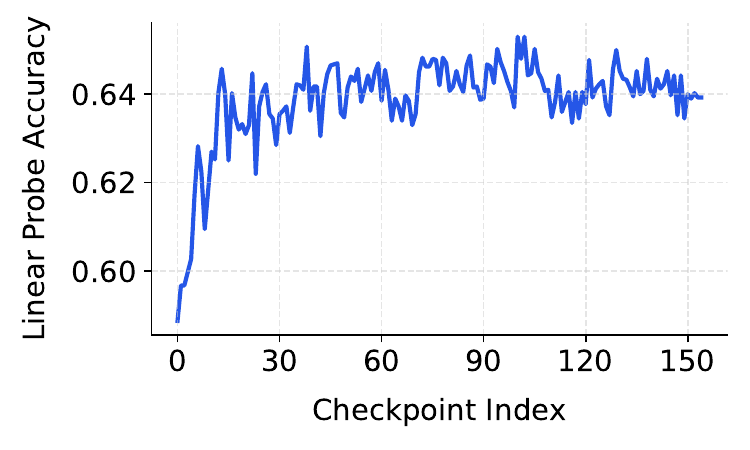}
    \caption{Pythia-6.9B}
\end{subfigure}
\begin{subfigure}{0.49\textwidth}
    \centering
    \includegraphics[width=\textwidth]{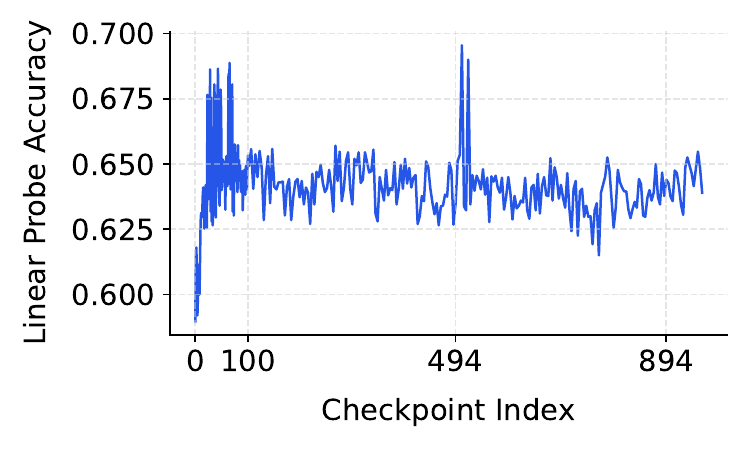}
    \caption{OLMo-7B}
\end{subfigure}
\caption{Last-layer linear probe accuracy of each weak-to-strong model on the HH-RLHF dataset during pre-training.}
\label{fig:linear_probe_helpful}
\end{figure*}

\begin{figure*}[t]
\centering
\begin{subfigure}{0.49\textwidth}
    \centering
    \includegraphics[width=\textwidth]{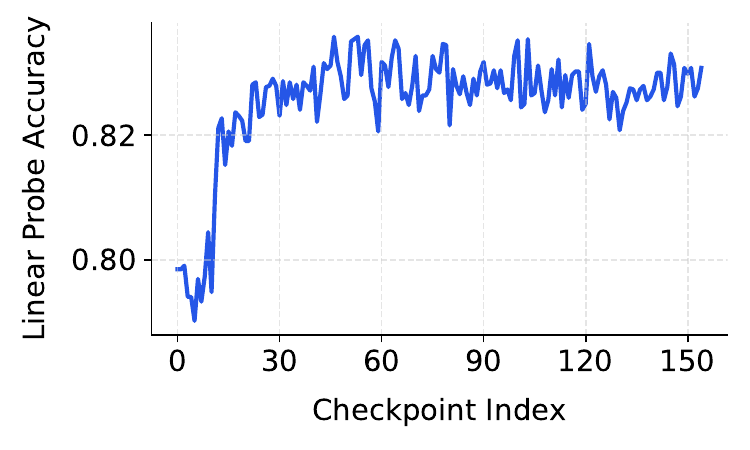}
    \caption{Pythia-6.9B}
\end{subfigure}
\begin{subfigure}{0.49\textwidth}
    \centering
    \includegraphics[width=\textwidth]{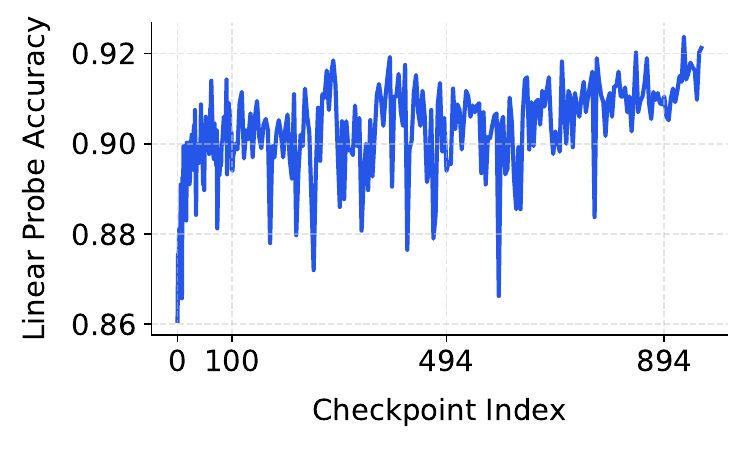}
    \caption{OLMo-7B}
\end{subfigure}
\caption{Last-layer linear probe accuracy of each weak-to-strong model on the CAI-Harmless dataset during pre-training.}
\vspace{-10pt}
\label{fig:linear_probe_cai}
\end{figure*}


\begin{figure*}[t]
\centering
\begin{subfigure}{0.49\textwidth}
    \centering
    \includegraphics[width=\textwidth]{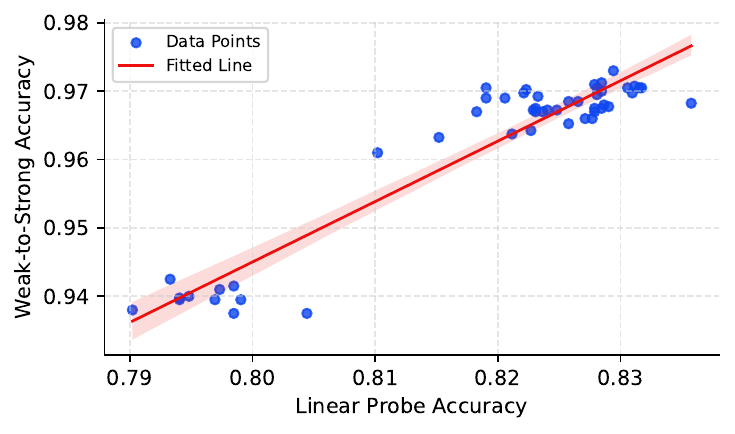}
    \caption{Pythia-6.9B}
\end{subfigure}
\begin{subfigure}{0.49\textwidth}
    \centering
    \includegraphics[width=\textwidth]{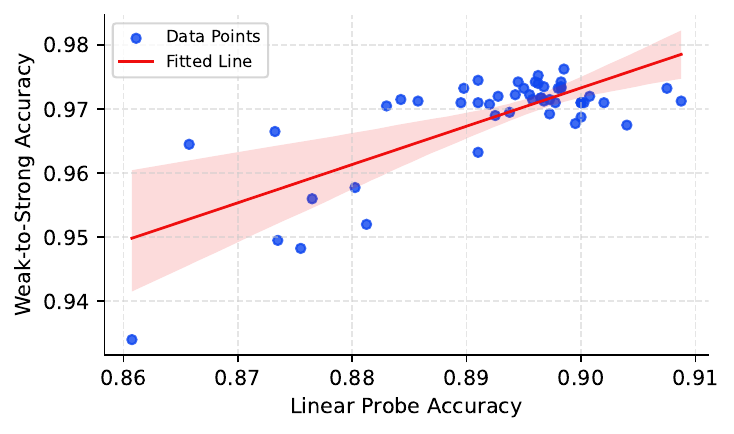}
    \caption{OLMo-7B}
\end{subfigure}
\caption{Correlation between linear probe accuracy and weak-to-strong accuracy on the CAI-Harmless dataset during the first 50 checkpoints. (a) $R^2=0.8983$, Spearman $\rho=0.7552$, $\text{p-value}=1.79 \times 10^{-25}$. (b) $R^2=0.5595$, Spearman $\rho=0.3947$, $\text{p-value}=4.30 \times 10^{-10}$.}
\vspace{-10pt}
\label{fig:correlation_cai}
\end{figure*}

\begin{figure*}[t]
\centering
\begin{subfigure}{0.49\textwidth}
    \centering
    \includegraphics[width=\textwidth]{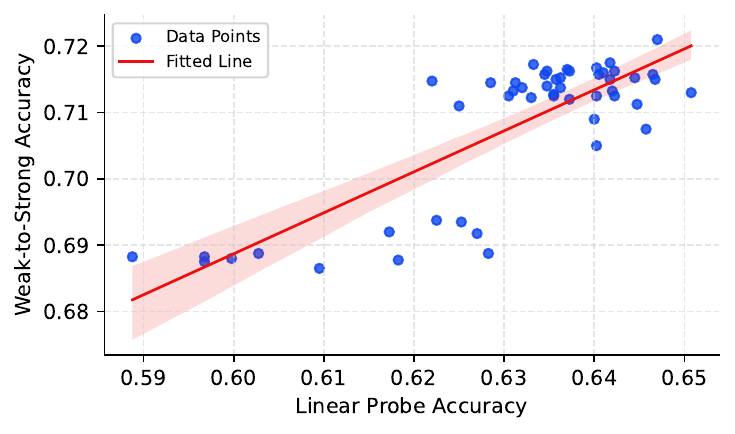}
    \caption{Pythia-6.9B}
\end{subfigure}
\begin{subfigure}{0.49\textwidth}
    \centering
    \includegraphics[width=\textwidth]{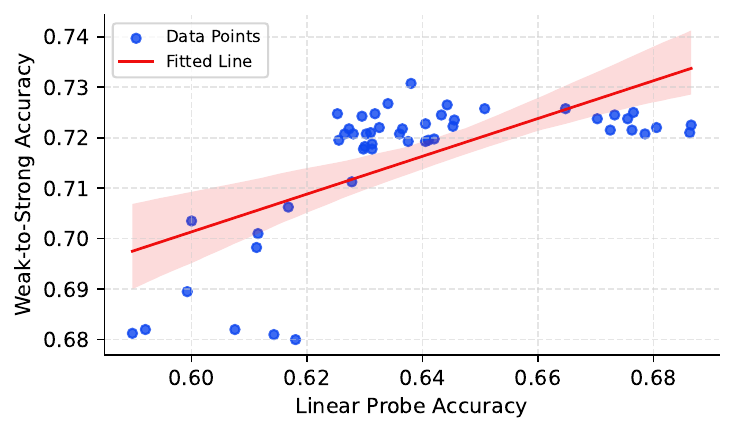}
    \caption{OLMo-7B}
\end{subfigure}
\caption{Correlation between linear probe accuracy and weak-to-strong accuracy on the HH-RLHF dataset during the first 50 checkpoints. The shaded region indicates standard deviation. (a) $R^2=0.6624$, Spearman $\rho=0.6003$, $\text{p-value}=6.67 \times 10^{-13}$. (b) $R^2=0.4403$, Spearman $\rho=0.663$, $\text{p-value}=1.5 \times 10^{-7}$.}
\vspace{-10pt}
\label{fig:correlation_helpful}
\end{figure*}


\begin{table}[t] 
\centering
\caption{Correlation analysis between W2SG performance and linear probe accuracy.}
\label{tab:correlation_stats}
\begin{tabular}{lccc}
\toprule
\multicolumn{1}{c}{\textbf{Model}} & \textbf{$R^2$} & \textbf{Spearman $\rho$} & \textbf{p-value} \\
\midrule
Pythia (HH-RLHF)      & $0.6624$ & $0.6003$ & $6.676 \times 10^{-13}$ \\
Pythia (CAI-Harmless) & $0.8983$ & $0.7552$ & $1.790 \times 10^{-25}$ \\
OLMo (HH-RLHF)        & $0.4403$ & $0.6630$ & $1.500 \times 10^{-7}$ \phantom{0} \\
OLMo (CAI-Harmless)   & $0.5595$ & $0.3947$ & $4.304 \times 10^{-10}$ \\
\bottomrule
\end{tabular}
\end{table}

\begin{table}[t] 
\centering
\caption{Correlation analysis between W2SG performance and evaluation loss.}
\label{tab:correlation_new}
\begin{tabular}{lccc}
\toprule
\multicolumn{1}{c}{\textbf{Model}} & \textbf{$R^2$} & \textbf{Spearman $\rho$} & \textbf{p-value} \\
\midrule
Pythia (HH-RLHF)          & $0.3697$ & $-0.5122$ & $3.647 \times 10^{-4}$ \\
Pythia (CAI-Harmless)     & $0.2610$ & $-0.6210$ & $1.503 \times 10^{-4}$ \\
OLMo (HH-RLHF)            & $0.4484$ & $-0.4910$ & $1.052 \times 10^{-7}$ \\
OLMo (CAI-Harmless)       & $0.3485$ & $-0.4670$ & $6.403 \times 10^{-6}$ \\
\bottomrule
\end{tabular}
\end{table}

\end{document}